%% file: main.tex
\documentclass{article}

\usepackage[dblblindworkshop, final]{neurips_2026}

\workshoptitle{TS-LIMITS}

\usepackage[utf8]{inputenc} %
\usepackage[T1]{fontenc}    %
\usepackage{hyperref}       %
\hypersetup{hidelinks}
\usepackage{url}            %
\usepackage{booktabs}       %
\usepackage{array}
\usepackage{amsfonts}       %
\usepackage{amsmath}
\usepackage{amsthm}
\usepackage{nicefrac}       %
\usepackage{microtype}      %
\usepackage{tabularx}
\usepackage{graphicx}
\usepackage{subcaption}
\usepackage{placeins}
\usepackage{xcolor}         %
\usepackage{tikz}

\newtheorem{proposition}{Proposition}

\IfFileExists{macros.tex}{\input{macros}}{}

\title{StaFIR: Convex Learning of Stationarity-Aware Causal Filters
}

\author{%
  Lorena Egger \\
  Research Engineer at Drakai Capital \\
  Paris \\
  \texttt{lorena.egger@drakaicapital.com} \\
  \And
  Mathis Linger \\
  CTO at Drakai Capital \\
  Paris \\
  \texttt{mathis.linger@drakaicapital.com} \\
}

\begin{document}

\raggedbottom

\maketitle

\input{sections/abstract}

\input{sections/introduction}
\input{sections/related_work}

\input{sections/methods}
\input{sections/experiments}
\input{sections/conclusion}

\clearpage

\bibliographystyle{unsrtnat}
\bibliography{bibliography/references}

\clearpage

\input{sections/appendix}

\end{document}

%% file: macros.tex
\newcommand{\Sstat}{S_{\mathrm{stat}}}

%% file: sections/abstract.tex
\begin{abstract}
Reducing nonstationarity in a persistent time series entails deciding how much of its temporal dependence to remove. In finance, fractional differencing is often tuned using the Augmented Dickey--Fuller (ADF) test, limiting the search to a one-parameter family of lag profiles and addressing input preservation only indirectly. We propose StaFIR, a causal finite-impulse-response filter with a learned nonnegative mixture of exponential lag profiles. Its convex learning objective balances empirical stationarity with similarity to the input. We evaluate StaFIR on ARFIMA--GARCH controlled settings and rolling financial series, including a realized-volatility forecasting task. The experiments show that StaFIR adjusts its filtering strength to persistence while limiting unnecessary transformation in stationary regimes. In downstream forecasting, there is no clear accuracy difference from fixed half-order differencing, while StaFIR achieves higher measured similarity to the raw signal. A complementary direct forecasting experiment finds that greater input similarity is associated with smaller forecasting penalties, although the raw representation remains stronger.
\end{abstract}

%% file: sections/introduction.tex
\section{Introduction}

Financial time series are commonly transformed before model training to reduce nonstationarity and improve out-of-sample performance \citep{hasslerPohle2023forecasting, DePrado2019}. However, financial markets are subject to distribution shifts and regime changes, causing the statistical properties of a time series to evolve over time. Hence, a transformation calibrated in one period can become unsuitable in another.

The choice of transformation must balance reducing nonstationarity against preserving temporal dependence that may remain useful for prediction.
Forecasting methods that restore instance statistics or temporal information after stationarization highlight the value of retaining access to information altered by preprocessing \citep{KimKTPCC22,liu2022nonstationary}. This motivates examining input preservation alongside empirical stationarity.
First-order differencing can remove stochastic trends, but may also discard persistent, low-frequency structure \citep{hosking1981fractional}. Fractional differencing generalizes conventional differencing to non-integer orders, allowing non-stationary processes to be transformed while preserving long-term dependence \citep{grangerFracDiff1980}. A common financial application selects the smallest fractional-differencing order for which the transformed series rejects the Augmented Dickey--Fuller (ADF) unit-root null, thereby limiting the loss of temporal dependence relative to first differencing \citep{DePrado2019}. However, this rule controls the balance between persistence reduction and preservation of temporal structure only indirectly, by minimizing the differencing order subject to a binary test decision. Explicitly assessing both properties would allow selection according to a continuous trade-off.

We propose StaFIR, a stationarity-aware finite-impulse-response (FIR) filter learned directly from calibration data. StaFIR is causal and preserves the sign structure of a differencing operator. The filter represents its lag profile using a nonnegative mixture of normalized exponential decays. Unlike a one-parameter fractional kernel, this representation separates the total weight placed on past observations from its allocation across short and long decay scales.
The learning objective balances input fidelity with empirical stationarity penalties and reduces to a small convex second-order cone program (SOCP) that can be solved efficiently using standard convex solvers. The StaFIR implementation and experimental scripts are publicly available at \url{https://github.com/LorenaEgger/StaFIR}.

\paragraph{Contributions.}
We make two main contributions. First, we introduce StaFIR, a causal filter that learns filtering strength and its allocation across time scales through a convex objective balancing input fidelity with empirical stationarity penalties. A contextual validation rule selects the trade-off between empirical stationarity and input similarity. Second, we evaluate StaFIR against fixed and adaptive preprocessing baselines on synthetic and financial data. In realized-volatility forecasting with recursive reintegration, StaFIR retains higher measured input similarity than fixed fractional differencing of order 0.5, with no clear difference in forecast loss across the tested horizons. We further examine sensitivity to the forecasting model and quantify recalibration costs.

%% file: sections/related_work.tex
\section{Background}

\paragraph{Stationarity and stationarity tests}
\label{sec:related_work_stat}
A real-valued stochastic process $\{X_t\}$ is weakly (or second-order) stationary if it has finite second moments, a time-invariant mean, and an autocovariance that depends only on the lag:
\begin{equation*}
  \mathbb{E}[X_t]=\mu,
  \qquad
  \mathrm{Cov}(X_t,X_{t+h})=\gamma_X(h)
  \qquad \forall t,h,
\end{equation*}
where $\gamma_X(h)$ denotes the autocovariance function at lag $h$. Stationarity is commonly assessed using three complementary tests. The Augmented Dickey--Fuller (ADF) and Phillips–Perron (PP) tests assess unit-root null hypotheses \citep{said1984testing, phillips1988testing}, whereas KPSS assesses stationarity under a short-memory calibration \citep{kwiatkowski1992testing}. KPSS can reject covariance-stationary fractionally integrated alternatives with $0<d<0.5$~\citep{leeSchmidt1996kpss}. Consequently, rejection alone does not establish failure of weak stationarity. The local Whittle estimator provides a continuous diagnostic of persistence \citep{robinson1995gaussian}.

\paragraph{Nonstationarity mitigation}
Approaches to nonstationarity differ in which aspect of the series they modify or model. Preprocessing transformations include logarithms, detrending, moving averages, exponential moving averages, integer differencing, and fractional differencing. Decomposition methods such as STL \citep{cleveland1990stl}, wavelets \citep{Mallat1989wavelets,Nason2000waveletProcesses}, empirical mode decomposition \citep{Huang1998emd}, and state-space or Kalman filtering \citep{Kalman1960,DurbinKoopman2012} separate a signal into trend, seasonal, latent, or residual components. More recent machine-learning approaches learn representations designed to mitigate distribution shift, including reversible normalization and frequency-domain stationarity-aware representations \citep{KimKTPCC22,piao2026tifotimeinvariantfrequencyoperator}. We focus on causal preprocessing transformations applied directly to the observed series, particularly differencing-based approaches.

\paragraph{Gr\"unwald--Letnikov (GL) fractional differencing}

Fractional differencing can be expressed using the Grünwald--Letnikov (GL) expansion \citep{podlubny1999fractional}. The same formulation is used in ARFIMA models and in fractional-differencing preprocessing for financial time series \citep{DePrado2019}. For $d\in(0,1]$, let $B$ denote the one-step backshift operator, $BX_t=X_{t-1}$. The fractional difference is given by the binomial expansion
\begin{equation}
  (1-B)^dX_t = \sum_{k=0}^{\infty}h_k^{\mathrm{GL}}(d)X_{t-k},
  \qquad
  h_k^{\mathrm{GL}}(d)=(-1)^k\binom{d}{k}.
  \label{eq:fracdiff_gl}
\end{equation}
The expansion is causal and reduces to ordinary first differencing when $d=1$. For $0<d<1$, the coefficients on past observations are negative for $k\geq1$ and decay slowly in magnitude according to a power law. Consequently, fractional differencing can attenuate persistent low-frequency variation more gradually than first differencing while retaining more dependence across lags. Truncation, initialization, and efficient-approximation schemes for finite-width GL filters are studied in \citet{Rosas2025,Maitra_fracdiff_2025,OMAY2026117813,jensen2014}.

%% file: sections/methods.tex
\section{Methods}
\label{sec:methods}

\paragraph{The Stationarity-Aware FIR (StaFIR) filter}

Let $x_{1:T}=(x_1,\ldots,x_T)$ denote the observed time series, and let $K$ be the maximum lag. For $t>K$, we transform it with the causal finite-impulse-response (FIR) filter
\begin{equation*}
  y_t=\sum_{k=0}^{K}h_kx_{t-k}.
\end{equation*}
The $K+1$ filter coefficients weight the current observation and its previous $K$ lags. We restrict the current-observation coefficient to $h_0=1$ and the lag coefficients to $h_k\leq0$ for $k=1,\ldots,K$. The finite-horizon GL fractional-differencing filter of Eq.~\eqref{eq:fracdiff_gl} belongs to this sign-constrained class and has a lag profile determined by a single fractional order. StaFIR instead combines \(J\) exponentially decaying lag patterns with learned nonnegative weights. For $j=1,\ldots,J$, let $\tau_j$ denote a candidate decay scale and define $\rho_j=\exp(-1/\tau_j)$. The filter coefficients are
\begin{equation*}
   h_0(w) = 1, \qquad h_k(w) = -\sum_{j=1}^{J} w_j \frac{(1-\rho_j)\rho_j^{k-1}}{1-\rho_j^K}, \qquad k = 1,\dots,K,
\end{equation*}
with $w\geq0$ and $\sum_{j=1}^{J} w_j\leq1$. For each scale $\tau_j$, the factor multiplying \(w_j\) defines an exponentially decaying profile across the \(K\) past lags, normalized so that the magnitudes of its coefficients sum to one. The nonnegative weight $w_j$ controls the contribution of that decay scale. Consequently, the aggregate magnitude of the lag coefficients satisfies
\begin{equation*}
    \sum_{k=1}^{K} |h_k(w)|
    =
    \sum_{j=1}^{J} w_j
    \frac{1-\rho_j}{1-\rho_j^K}
    \sum_{k=1}^{K} \rho_j^{k-1}
    =
    \sum_{j=1}^{J} w_j
    \leq 1
    =
    h_0(w).
\end{equation*}
The value $\sum_{j=1}^{J} w_j$ also controls the strength of the transformation: if it is zero, nonnegativity implies $w=0$ and hence $y_t=x_t$, while larger values subtract more of the past from the current observation. The relative values of the $w_j$ determine the time scales over which this subtraction is distributed.

Over a finite FIR horizon, the GL kernel admits an exact continuous mixture of
normalized exponential profiles \citep{LeonenkoPodlubny2026}. StaFIR restricts
this representation to a fixed finite dictionary. Appendix~\ref{ap:survival_bridge}
proves the exact mixture identity and reports the reconstruction error using our nine-scale dictionary.
The dictionary is not restricted to GL shapes: learned weights can instead
concentrate on shorter decay scales and produce more rapidly decaying lag
profiles.

\paragraph{Learning the filter by convex optimization}

For an unstandardized input $\widetilde x_t$, define
$ x_t=\frac{\widetilde x_t-\widehat\mu_{\mathrm{tr}}}
{\widehat s_{\mathrm{tr}}}$,
where $\widehat\mu_{\mathrm{tr}}$ and $\widehat s_{\mathrm{tr}}$ are estimated from the corresponding training segment. Let $n=T-K$, let $y(w)\in\mathbb R^n$ collect the valid filtered outputs, and let $x=(x_{K+1},\ldots,x_T)$ denote the aligned input. Split the first $Im$ outputs into $I=\lfloor n/m\rfloor$ non-overlapping windows and write observation $t$ in window $i$ as $y_{i,t}(w)$.  
\begin{equation*}
\widehat v_i(w)
=\frac{1}{m}\sum_{t=1}^{m}
\bigl(y_{i,t}(w)-\bar y_i(w)\bigr)^2
\qquad
\widehat g_{i,\ell}(w)
=\frac{1}{m-\ell}\sum_{t=1}^{m-\ell}
\bigl(y_{i,t+\ell}(w)-y_{i,t}(w)\bigr)^2 .
\end{equation*}
For a set of positive lags $\mathcal L$, we combine these quantities into the worst-window local second-order stability penalty
\begin{equation}
  \mathcal R_{\mathrm{loc}}(w)
  =\max_{1\leq i\leq I}\widehat v_i(w)
  +\frac{1}{|\mathcal L|}\sum_{\ell\in\mathcal L}
    \max_{1\leq i\leq I}\widehat g_{i,\ell}(w).
  \label{eq:second-order-regularizer}
\end{equation}
The first term penalizes windows with large variation around their local mean, while the second penalizes large mean-squared differences between observations separated by $\ell$ time steps. Both controls are deliberately one-sided: they suppress unusually large local second-order variation but do not penalize windows for having unusually small variance or lagged differences. Consequently, $\mathcal R_{\mathrm{loc}}$ may miss nonstationarity expressed through shifts toward lower-variance regimes, but this one-sided design preserves convexity.

A second penalty discourages slow drift by penalizing the cumulative sums of the filtered output, which grow when the output diverges from zero:
\begin{equation}
  \mathcal R_{\mathrm{slow}}(w)
  =\frac{1}{n^2}\sum_{t=1}^{n}
    \left(\sum_{u=1}^{t}y_u(w)\right)^2 .
  \label{eq:kpss-regularizer}
\end{equation}
This partial-sum device is borrowed from the KPSS construction~\citep{kwiatkowski1992testing}. We use only the raw sums, without the long-run-variance normalization of the test statistic, because the resulting term is convex quadratic in $w$ (Appendix~\ref{ap:program_properties}).

For fixed penalty weights $\lambda_s,\lambda\geq0$, we learn the filter from the convex program
\begin{align}
  \underset{w}{\operatorname{minimize}}\quad
  &\underbrace{\frac{1}{n}\|y(w)-x\|_2^2}_{\text{input-fidelity loss}}
  +\lambda_s\underbrace{\mathcal R_{\mathrm{slow}}(w)}_{\text{slow drift control}}
  +\lambda\underbrace{\mathcal R_{\mathrm{loc}}(w)}_{\text{local second-order control}}
  \label{eq:socp-objective}\\
  \text{s.t.}\quad
  &w\geq0,\qquad \sum_{j=1}^{J} w_j\leq1.
  \label{eq:mass-cap}
\end{align}
The fidelity term discourages transformations stronger than necessary for our objective. The penalties weighted by $\lambda_s$ and $\lambda$ target persistent low-frequency movements and extreme local second-order behavior, respectively. Validation selects $(\lambda_s,\lambda)$, avoiding scale-dependent hard constraints. Because $h(w)$ is affine in $w$, the program of Eqs.~\eqref{eq:socp-objective}--\eqref{eq:mass-cap} is second-order cone representable for every fixed $(\lambda_s,\lambda)$. The epigraph formulation and proof can be found in Appendix~\ref{ap:socp_convexity}.

\paragraph{Validation scores and contextual selection}
\label{sec:scores}

We combine complementary measures into two continuous scores for validation and evaluation. This captures multiple aspects of empirical stationarity and input similarity, including discrepancies during extreme episodes and nonlinear dependence that overall linear correlation can miss. The stationarity score $\Sstat$ summarizes stability in local mean, variance, and lagged dependence, while the similarity score $S_{\mathrm{sim}}$ summarizes dependence between the transformed and original series, including their tails. Higher values indicate greater empirical stability and similarity over the evaluation window. We use “empirical stationarity” operationally to describe stability of local sample moments. These moments fluctuate even under weak stationarity, and reducing temporal dependence can improve the score without correcting time variation in population moments. Appendix~\ref{ap:score_definitions} gives the complete definitions.

We use a contextual validation criterion that favors configurations preserving similarity to the input while penalizing shortfalls below a data-dependent stationarity target. This target is calibrated from the stationarity scores of the raw series and its first differences, with \(\zeta\) specifying the fraction of the observed improvement to require. For StaFIR, we apply this criterion to the candidate pairs \((\lambda_s,\lambda)\) and refer to the resulting validation-selected procedure as \texttt{StaFIR-context}. Its formal specification is provided in Appendix~\ref{ap:empirical-rule}.

%% file: sections/experiments.tex
\section{Experiments and results}
\label{sec:experiments}
\label{sec:metrics}

We evaluate StaFIR in three settings: controlled long-memory processes,
rolling daily log prices of four financial series, and realized-volatility
forecasting on the same four markets. The market series are the S\&P 500, EUR/USD, a five-year
crossover credit index, and Bitcoin. Across all performance experiments, methods are configured 
using calibration data and evaluated on held-out data. Where applicable, hyperparameters are 
selected on a validation subset before refitting on the full calibration period. 
In rolling experiments, this procedure is repeated at each origin. Data spans, splits, 
candidate grids, and implementation details are reported in Appendix~\ref{ap:experimental_details}.

We compare StaFIR with the raw input, first differencing, an exponentially
weighted moving-average residual (\texttt{EWMA}), a rolling z-score, and three
fractional-differencing (FD) baselines. \texttt{FD-adf} 
selects the smallest $d$ for which the ADF test rejects the unit-root null \citep{DePrado2019}.
\texttt{FD-context} applies our contextual selection rule to fractional differencing, 
choosing $d$ from the candidate grid according to the validation series’ stationarity and similarity scores.
\texttt{FD(0.5)} fixes $d=0.5$, a robust choice for forecasting realized
volatility \citep{hasslerPohle2023forecasting}. For the downstream task, we additionally include \texttt{log-HAR}, a strong classical baseline for realized-volatility forecasting \citep{corsi2009har}. It predicts log realized variance from its current value and weekly and monthly averages, fitting a separate direct model at each forecast horizon. Appendices~\ref{ap:baseline_details}
and~\ref{ap:downstream-task} give the baseline details.

All FIR transformations use maximum lag $K=100$. For both \texttt{StaFIR-context} and \texttt{FD-context}, we fix the contextual target fraction at $\zeta=0.9$ in all main experiments. We restrict attention to
nonnegative fractional-integration orders, $d_{\mathrm{true}}\geq0$, i.e., persistent
 processes. For an output $y$ and its aligned input $x$, we report the stationarity gain
$\Delta S_{\mathrm{stat}}=S_{\mathrm{stat}}(y)-S_{\mathrm{stat}}(x)$ and the
similarity cost $C_{\mathrm{sim}}=100-S_{\mathrm{sim}}(x,y)$. Throughout the results, ''stationarity gain'' denotes improvement in this empirical score and does not by itself establish a reduction in population nonstationarity. Effect turnover
$T_y$ measures how much the realized transformation changes between consecutive
rolling origins. Appendix~\ref{ap:score_definitions} gives the score definitions
and effect turnover.

\subsection{Adaptation to the level of memory on synthetic data}

We generate
$\operatorname{ARFIMA}(0,d_{\mathrm{true}},0)$--%
$\operatorname{GARCH}(1,1)$ series \citep{ling1997fractionally}, varying values of the
fractional-integration order $d_{\mathrm{true}}$ and the volatility persistence
$p_{\mathrm{garch}}$. These processes are covariance-stationary for
$d_{\mathrm{true}}<0.5$ and nonstationary for $d_{\mathrm{true}}\geq0.5$ \citep{hosking1981fractional}, and
$p_{\mathrm{garch}}$ independently controls volatility persistence. The design
tests two things: whether a method transforms persistent inputs sufficiently,
and whether it avoids unnecessary transformation below the stationarity
boundary. See Appendix~\ref{ap:synthetic_data_generation} for more details. Figure~\ref{fig:synthetic-adaptivity} summarizes both.

\begin{figure*}[t]
  \centering
  \includegraphics[width=0.98\textwidth]{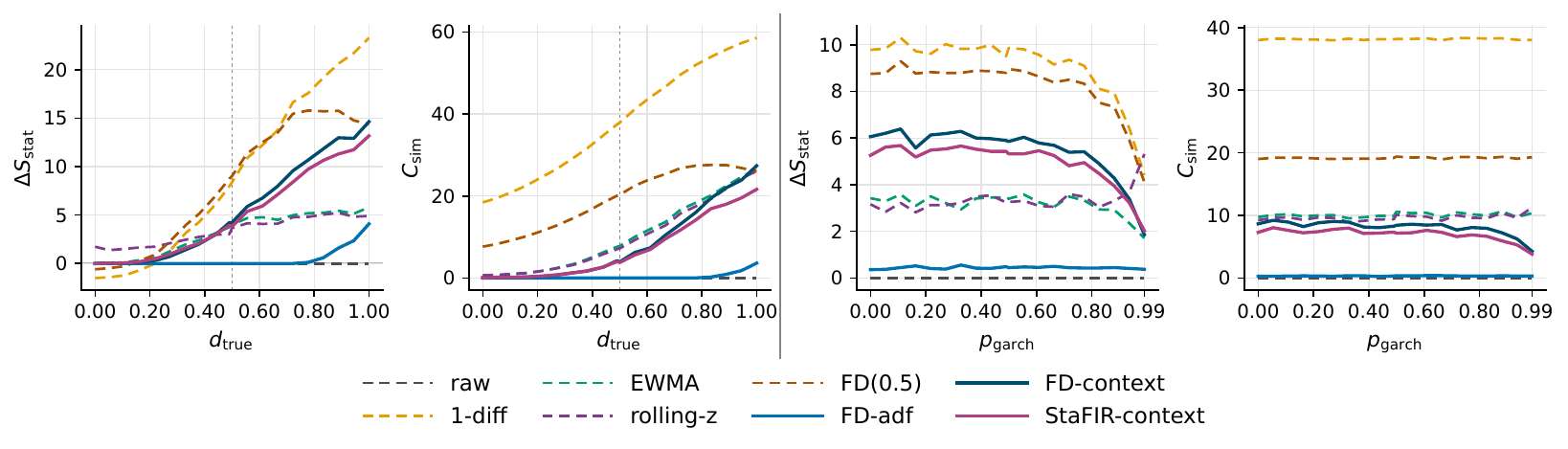}
  \caption{Held-out stationarity gain $\Delta S_{\mathrm{stat}}$ and similarity
  cost $C_{\mathrm{sim}}$ on the synthetic designs. Dashed lines indicate fixed
  methods and solid lines adaptive methods. The two panels left of the separator vary the
  fractional-integration order
  $d_{\mathrm{true}}$ on both sides of the stationarity boundary
  $d_{\mathrm{true}}=0.5$, marked by the fine vertical line. The two panels
  to the right vary GARCH persistence over
  $p_{\mathrm{garch}}\in[0,0.99]$. Lines average over 30 seeds.}
  \label{fig:synthetic-adaptivity}
\end{figure*}

Above the boundary, both contextual methods transform more strongly as
$d_{\mathrm{true}}$ increases, whereas stronger volatility persistence reduces
their stationarity gains. Averaged over the nonstationary grid,
\texttt{StaFIR-context} gains 7.86 stationarity points at a similarity cost of
11.80, compared with 8.70 and 14.00 for \texttt{FD-context}. First
differencing produces the largest stationarity gain but also the largest loss
of signal. Below the boundary, \texttt{StaFIR-context} and \texttt{FD-context}
stay close to the identity, whereas \texttt{FD(0.5)} keeps paying similarity
cost, in accordance with its lack of adaptivity. We test a simple regime shift scenario where $d_{\mathrm{true}}$ or 
$p_{\mathrm{garch}}$ changes at a scheduled point within a series: 
the two contextual methods adjust similarly in the same
direction. The full score table and an additional regime shift experiment
for \texttt{StaFIR-context} and \texttt{FD-context} are
reported in Appendices~\ref{ap:synthetic-additional-results}
and~\ref{ap:regime-shift-results} respectively.

\subsection{Stationarity--similarity trade-off on rolling log prices}

We apply the rolling-window protocol defined in Appendix~\ref{ap:rolling-protocol} to the daily log prices of the four
financial series. Figure~\ref{fig:real-levels-tradeoff} places every method in
the plane spanned by similarity cost and empirical stationarity gain and traces the path of fixed-order fractional differencing as $d$ varies for reference. 
Table~\ref{tab:real-tests-main} reports conventional held-out diagnostics on the same outputs: ADF, PP, and KPSS results, and the local Whittle estimate of the residual memory.

\begin{figure*}[t]
  \centering
  \includegraphics[width=0.95\textwidth]{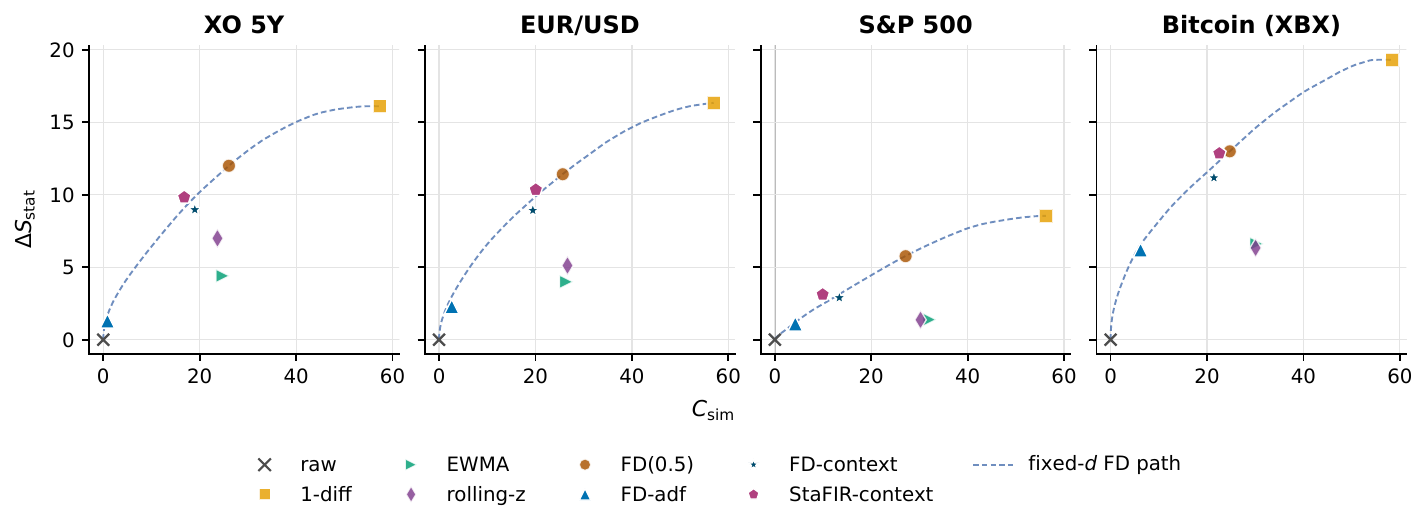}
  \caption{Mean held-out stationarity gain $\Delta S_{\mathrm{stat}}$ and
  similarity cost $C_{\mathrm{sim}}$ across rolling
  origins for the four log-price series, on shared axes. The dashed curve is
  the path traced by fixed-order fractional differencing as $d$ varies. Stationarity gain is the
  selection criterion of the contextual methods; Table~\ref{tab:real-tests-main}
  gives additional held-out diagnostics, and Appendix~\ref{ap:real-paired-main} gives
  paired bootstrap intervals.}
  \label{fig:real-levels-tradeoff}
\end{figure*}

\input{generated/tables/real_tests_main}

On the equal-weight four-series panel, \texttt{StaFIR-context} gains 1.05 more
stationarity points than \texttt{FD-context} (95\% interval $[0.39,1.71]$)
with no clear difference in similarity cost. It also rejects ADF and PP unit-root nulls more often and leaves less residual memory: the paired difference in
absolute local Whittle memory is $-0.08$ with interval $[-0.11,-0.05]$.
However, \texttt{StaFIR-context} also has the lowest KPSS pooled non-rejection rate.
The fixed \texttt{FD(0.5)} transformation
rejects the unit-root nulls at every origin but has lower measured input similarity. Across rolling origins, 
\texttt{StaFIR-context} exhibits material changes in its realized transformation more often than \texttt{FD-context}, 
but the median magnitude of these changes is smaller. Per-series diagnostics, paired
uncertainty, and turnover results are reported in
Appendix~\ref{ap:real-additional-results}. Appendix Table~\ref{tab:stationary-input-diagnostics} compares the mean $S_{\mathrm{stat}}$ and KPSS non-rejection percentages on raw synthetic inputs in the stationary parameter range to further study their differences. Appendix~\ref{ap:selection-sensitivity} examines sensitivity to the target fraction, shortfall penalty, and switching cost.

\subsection{Downstream result: realized-volatility forecasting}

The downstream experiment forecasts the log of a realized-variance target
computed from the trailing five daily returns of the same four markets. In the
main experiment, each FIR representation uses the same ten-lag ridge
autoregression, which predicts the next transformed value as a linear
function of its ten most recent values.
These forecasts are iterated recursively to
\(h\in\{1,5,20,40,80\}\), and the fitted filter is reintegrated over each
forecast path to obtain \(\widehat\ell_{t+h\mid t}\) on the unfiltered
log-realized-variance scale. This implies the positive forecast
\(\widehat V_{t+h\mid t}=\exp(\widehat\ell_{t+h\mid t})\). For realized variance \(V_{t+h}\), we evaluate forecasts using the
quasi-likelihood (QLIKE) loss \citep{patton2011volatility}. We report each method's percentage change in
mean QLIKE relative to \texttt{raw}, averaged equally across the four
series. Appendix~\ref{ap:downstream-task} gives the complete forecasting
and uncertainty protocol and QLIKE definition, while Appendix~\ref{ap:program_properties}
establishes the inversion conditions and describes the treatment of the
unit-root boundary.

\input{generated/tables/downstream_results}

\texttt{StaFIR-context} has lower mean QLIKE than the raw representation at all
five horizons, with improvements between 0.83\% and 3.28\%. Paired intervals
exclude zero through horizon 40 but not at horizon 80. It is the only adaptive
method whose point estimates improve on \texttt{raw} at all five horizons.
Direct comparisons with \texttt{FD(0.5)} include zero at every horizon, so we
find no clear forecasting difference between the two methods. \texttt{StaFIR-context}
nevertheless has higher measured input similarity, with similarity scores
of 93.5 and 79.4, respectively. The per-series results in
Appendix~\ref{ap:downstream-additional-results} show that the gains are
concentrated on EUR/USD and small elsewhere. The external \texttt{log-HAR}
reference has the strongest point estimates at horizons 5--40, but
\texttt{StaFIR-context} reverses that ranking at $h=1$ and $h=80$. Because ten consecutive outputs of the 101-coefficient FIR depend on
110 raw observations, we also compare against a raw AR(110) control
with the same available history and training protocol. Relative to
this control, \texttt{StaFIR-context} reduces panel QLIKE by 3.92\%
and 8.24\% at $h=1$ and $h=5$, respectively, with marginal 95\%
intervals excluding zero; intervals at longer horizons include zero
(Appendix~\ref{ap:downstream-additional-results},
Table~\ref{tab:downstream-direct-comparisons}).
This supports a benefit beyond access to additional history.

As a sensitivity check, we also fit a separate ten-lag ridge head for
each horizon to predict future raw log realized variance directly from
transformed lags, without recursion or reintegration. Under this direct
head, \texttt{StaFIR-context} outperforms \texttt{FD(0.5)} at
$h=1,5,20$ but trails \texttt{raw} at $h=1,5,20,40$
(Appendix~\ref{ap:direct-head-sensitivity},
Table~\ref{tab:downstream-direct-head-sensitivity}).
These results indicate that relative forecasting performance depends
on how the filtered representation is used.

\subsection{Deployment limits: inference and adaptation cost}

All adaptive transformations in the deployment benchmark share a 101-coefficient FIR inference path, with median latency of 4.0~$\mu$s and 1.59~KiB of persistent numeric state. Computational differences arise during adaptation: \texttt{StaFIR-context} fits and scores 98 candidates per selection origin, taking 1.56~s on one worker, compared with 0.12~s for \texttt{FD-context}.

Selecting $(\lambda_s,\lambda)$ every eighth origin lowers the amortized adaptation time to 0.37~s per origin, while filter weights are still re-fitted at every origin. Paired forecasting comparisons favor selection at every origin over several less frequent schedules at some horizons, but do not establish a monotonic benefit from increasing selection frequency. Appendix~\ref{ap:adaptation-cost} reports the full costs and selection-schedule comparisons.

%% file: generated/tables/real_tests_main.tex
\begin{table}[t]
  \caption{Conventional held-out stationarity diagnostics on the rolling
  log-price series. Entries are pooled rates (\%) of rejection of the ADF and PP
  unit-root nulls and non-rejection of the KPSS stationarity null.
  LW $\widehat d$ is the local Whittle estimate of the
  residual memory of the held-out output, with a 95\% stationary-bootstrap
  interval. The raw series pool at $\widehat d=0.96$ and first
  differences at $-0.04$. Estimation details are in
  Appendix~\ref{ap:real-additional-results}.}
  \label{tab:real-tests-main}
  \centering
  \scriptsize
  \setlength{\tabcolsep}{7pt}
  \begin{tabular}{@{}lcccc@{}}
    \toprule
    Method & \shortstack{ADF\\rejection} & \shortstack{PP\\rejection} & \shortstack{KPSS\\non-rejection} & LW $\widehat d$ \\
    \midrule
    \texttt{FD(0.5)} & 100 & 100 & 31 & 0.46 [0.44, 0.48] \\
    \texttt{FD-adf} & 38 & 43 & 11 & 0.82 [0.80, 0.85] \\
    \texttt{FD-context} & 71 & 72 & 26 & 0.61 [0.56, 0.65] \\
    \texttt{StaFIR-context} & 77 & 80 & 10 & 0.53 [0.49, 0.58] \\
    \bottomrule
  \end{tabular}
\end{table}

%% file: generated/tables/downstream_results.tex
\begin{table*}[t]
  \caption{Downstream QLIKE relative to \texttt{raw} (\%, equal-weight mean
  over four series $\pm$ one paired stationary-bootstrap standard error, with
  bootstrap details in Appendix~\ref{ap:downstream-task}). Lower is
  better. Bold marks the lowest point estimate at each horizon. \texttt{FD-adf} equals \texttt{raw} because it selects $d=0$ throughout. Similarity is measured on the realized-variance inputs.
  All preprocessing rows use the common forecasting head. log-HAR is an external task-specific reference with a different direct forecasting architecture.
  }
  \label{tab:downstream-results}
  \centering
  \scriptsize
  \setlength{\tabcolsep}{5pt}
  \begin{tabular}{@{}lrrrrrr@{}}
    \toprule
    Method & Similarity & $h=1$ & $h=5$ & $h=20$ & $h=40$ & $h=80$ \\
    \midrule
    \texttt{raw} $=$ \texttt{FD-adf} & 100.0 & $0.00$ & $0.00$ & $0.00$ & $0.00$ & $0.00$ \\
    \texttt{1-diff} & 53.9 & $6.13\pm0.89$ & $27.12\pm4.92$ & $42.84\pm10.75$ & $49.19\pm10.34$ & $87.01\pm18.43$ \\
    \texttt{EWMA} & 91.8 & $1.32\pm0.60$ & $8.54\pm2.78$ & $13.53\pm4.86$ & $21.71\pm7.01$ & $31.82\pm7.44$ \\
    \texttt{FD(0.5)} & 79.4 & $\mathbf{-0.86\pm0.55}$ & $0.77\pm1.93$ & $-1.83\pm2.90$ & $0.06\pm3.67$ & $0.94\pm2.85$ \\
    \texttt{FD-context} & 93.9 & $-0.34\pm0.34$ & $0.21\pm1.10$ & $-0.80\pm1.45$ & $0.34\pm1.79$ & $0.97\pm1.25$ \\
    \texttt{StaFIR-context} & 93.5 & $-0.83\pm0.30$ & $-2.25\pm0.81$ & $-3.28\pm1.09$ & $-3.07\pm1.38$ & $\mathbf{-1.84\pm1.33}$ \\
    \midrule
    \texttt{log-HAR} & -- & $2.85\pm0.79$ & $\mathbf{-2.92\pm1.58}$ & $\mathbf{-3.93\pm2.12}$ & $\mathbf{-3.72\pm2.19}$ & $-1.17\pm1.77$ \\
    \bottomrule
  \end{tabular}
\end{table*}

%% file: sections/conclusion.tex
\section{Conclusion}

We introduced StaFIR, a causal FIR filter that learns filtering strength and its allocation across time scales through convex optimization. Synthetic experiments show adaptation to persistence. On rolling financial series, \texttt{StaFIR-context} improves the empirical stationarity score relative to \texttt{FD-context}, with more frequent rejection of the ADF and PP unit-root nulls and lower local Whittle persistence estimates, but less frequent KPSS non-rejection. With recursive reintegration, \texttt{StaFIR-context} shows no clear forecasting difference from \texttt{FD(0.5)} while retaining higher measured input similarity. With a direct forecasting head, \texttt{StaFIR-context} outperforms \texttt{FD(0.5)} at shorter horizons but generally trails \texttt{raw} inputs. Future work should test whether higher measured input similarity translates into improved predictive performance on other tasks, including under tail-sensitive losses. Further studies could assess different filter lengths and smaller decay-scale dictionaries. Methodological extensions could explore penalties that capture both increases and decreases in local second moments, addressing the current penalty’s one-sided treatment of local variation.

%% file: sections/appendix.tex
\appendix

\section*{Appendix}

The appendix is organized as follows. Appendix~\ref{ap:survival_bridge}
connects GL fractional differencing to the exponential mixture through the
survival-curve representation and an exact beta-integral identity.
Appendix~\ref{ap:program_properties} establishes the properties of the
estimation program, including its second-order cone form and its
invertibility. Appendix~\ref{ap:scores_selection} defines the evaluation
scores, effect turnover, and the contextual selection rule.
Appendix~\ref{ap:experimental_details} records the compared methods, the
rolling protocol, the candidate grids and solver settings, the data, and the
forecasting task. Appendix~\ref{ap:additional-results}
reports additional results in the order of Section~\ref{sec:experiments},
followed by ablations, sensitivity analyses, and the adaptation cost and its
measurement.

\section{The Survival-Curve Bridge from GL to the Exponential Mixture}
\label{ap:derivations}
\label{ap:survival_bridge}

Every differencing-type filter of Section~\ref{sec:methods} is generated from
a nonincreasing survival curve of lag weight. Let
$1\geq a_0\geq\cdots\geq a_K=0$ and set $h_k=-(a_{k-1}-a_k)$ for
$k=1,\ldots,K$. Here $a_0$ is the total weight placed on past observations,
$a_k$ is the part of that weight assigned to lags beyond $k$, and $|h_k|$ is
the weight at lag $k$. Writing $a_k=a_0 s_k$, where $s_k=a_k/a_0$ is the
normalized survival shape with $s_0=1$ and $s_K=0$, separates each filter into
a total lag weight and a survival shape. This section expresses the GL kernel
and the proposed mixture in these shared coordinates and then connects them
directly through Euler's beta integral.

\paragraph{GL in survival coordinates.}
A recursive implementation of the GL weights in Eq.~\eqref{eq:fracdiff_gl} is
\begin{equation}
  h_k^{\mathrm{GL}}(d) =
  \begin{cases}
    1, & k = 0, \\
    -h_{k-1}^{\mathrm{GL}}(d)\dfrac{d-k+1}{k}, & k \ge 1.
  \end{cases}
  \label{eq:fracdiff_rec}
\end{equation}
For $0<d<1$, the GL coefficients $h_k^{\mathrm{GL}}(d)=(-1)^k\binom{d}{k}$,
$k\ge1$, are the negated increments of the survival sequence below, where
$\Gamma(\cdot)$ denotes Euler's gamma function:
\[
   s_k^{\mathrm{GL}}(d)
   =
   \frac{\Gamma(k+1-d)}{\Gamma(1-d)\Gamma(k+1)},
   \qquad
   h_k^{\mathrm{GL}}(d)
   =
   -\left(s_{k-1}^{\mathrm{GL}}(d)-s_k^{\mathrm{GL}}(d)\right)
   =
   -\frac{d\,\Gamma(k-d)}{\Gamma(1-d)\Gamma(k+1)}.
\]
The increment identity follows from two applications of the recurrence
$\Gamma(z+1)=z\Gamma(z)$, to $\Gamma(k+1)$ and to $\Gamma(k+1-d)$. As
$k\to\infty$,
\[
   s_k^{\mathrm{GL}}(d)\sim \frac{k^{-d}}{\Gamma(1-d)},
   \qquad
   \left|h_k^{\mathrm{GL}}(d)\right|\sim \frac{d\,k^{-d-1}}{\Gamma(1-d)}.
\]
Thus the GL survival sequence decays as $k^{-d}$, whereas the magnitudes of
its filter coefficients decay as $k^{-d-1}$. For a finite horizon $K$, let
$m_K=1-s_K^{\mathrm{GL}}(d)$ and define the normalized shape
$\widetilde s_k^{\mathrm{GL}}=(s_k^{\mathrm{GL}}-s_K^{\mathrm{GL}})/m_K$ for
$k=0,\ldots,K$. Then $\widetilde s_0^{\mathrm{GL}}=1$,
$\widetilde s_K^{\mathrm{GL}}=0$, and
$h_k^{\mathrm{GL}}=-m_K(\widetilde s_{k-1}^{\mathrm{GL}}-\widetilde s_k^{\mathrm{GL}})$.
This is the survival parameterization above, with total lag weight $a_0=m_K$
and normalized survival shape $\widetilde s^{\mathrm{GL}}$.

\paragraph{The mixture in survival coordinates.}
Fix the candidate decay scales $\tau_1,\ldots,\tau_J$ of
Section~\ref{sec:methods} with $\rho_j=\exp(-1/\tau_j)$, and let $p$ be a
nonnegative weight vector summing to one that allocates the survival shape
across those scales. The normalized finite-horizon survival shape of the
mixture family is
\begin{equation}
   s_k^{\mathrm{StaFIR}}(p) = \sum_{j=1}^{J} p_j \frac{\rho_j^k-\rho_j^K}{1-\rho_j^K}, \qquad k=0,\ldots,K.
  \label{eq:stafir-survival}
\end{equation}
The estimation program optimizes the unnormalized nonnegative weights $w$
directly, so the total lag weight and the survival shape are learned jointly.
When $\sum_{j=1}^{J} w_j>0$, the normalized survival shape is recovered as
$p=w/\sum_{j=1}^{J} w_j$, and the coefficients of Section~\ref{sec:methods}
satisfy
$h_k(w)=-\sum_{j=1}^{J} w_j\left(s_{k-1}^{\mathrm{StaFIR}}(p)-s_k^{\mathrm{StaFIR}}(p)\right)$
for $k\geq1$. When $w=0$, the transformation is the identity and $p$ is
immaterial. In survival coordinates, the mixture family therefore has total
lag weight $\sum_{j=1}^{J} w_j$ and a learned exponential-mixture survival shape.

\paragraph{Exact beta-mixture representation.}
The connection follows directly from the beta integral.

\begin{proposition}
For $0<d<1$, define the probability measure
\begin{equation*}
  \mu_d(\mathrm d\rho)
  =
  \frac{\rho^{-d}(1-\rho)^{d-1}}{\mathrm B(1-d,d)}\,\mathrm d\rho,
  \qquad 0<\rho<1,
\end{equation*}
where $\mathrm B(a,b)=\Gamma(a)\Gamma(b)/\Gamma(a+b)$ is Euler's beta
function. Then
\begin{equation}
  s_k^{\mathrm{GL}}(d)
  = \int_0^1 \rho^k\,\mu_d(\mathrm d\rho),
  \qquad k\geq0.
  \label{eq:gl-beta-survival}
\end{equation}
For a finite horizon $K$, let
\begin{equation}
  W_{d,K}(\mathrm d\rho)
  = (1-\rho^K)\mu_d(\mathrm d\rho).
  \label{eq:gl-finite-mixture-measure}
\end{equation}
This is a nonnegative finite measure with total mass $m_K$, and the truncated
GL coefficient magnitudes satisfy
\begin{equation}
  \left|h_k^{\mathrm{GL}}(d)\right|
  =
  \int_0^1
  \frac{(1-\rho)\rho^{k-1}}{1-\rho^K}
  \,W_{d,K}(\mathrm d\rho),
  \qquad k=1,\ldots,K.
  \label{eq:gl-beta-coefficients}
\end{equation}
Equivalently, with $\nu_{d,K}=W_{d,K}/m_K$, the normalized finite-horizon
survival shape satisfies
\begin{equation}
  \widetilde s_k^{\mathrm{GL}}(d)
  =
  \int_0^1
  \frac{\rho^k-\rho^K}{1-\rho^K}
  \,\nu_{d,K}(\mathrm d\rho),
  \qquad k=0,\ldots,K.
  \label{eq:gl-beta-finite-survival}
\end{equation}
\end{proposition}

\begin{proof}
Euler's beta integral gives
\[
  \int_0^1 \rho^k\,\mu_d(\mathrm d\rho)
  =
  \frac{\mathrm B(k+1-d,d)}{\mathrm B(1-d,d)}
  =
  \frac{\Gamma(k+1-d)}{\Gamma(1-d)\Gamma(k+1)}
  = s_k^{\mathrm{GL}}(d),
\]
which proves Eq.~\eqref{eq:gl-beta-survival}. Taking successive differences
therefore yields
\[
  \left|h_k^{\mathrm{GL}}(d)\right|
  = s_{k-1}^{\mathrm{GL}}(d)-s_k^{\mathrm{GL}}(d)
  = \int_0^1 (1-\rho)\rho^{k-1}\,\mu_d(\mathrm d\rho).
\]
Moreover,
\[
  \int_0^1 W_{d,K}(\mathrm d\rho)
  = 1-s_K^{\mathrm{GL}}(d)
  = m_K < 1.
\]
Substitution of Eq.~\eqref{eq:gl-finite-mixture-measure} proves
Eq.~\eqref{eq:gl-beta-coefficients}. Summing the coefficient representation
over lags beyond $k$ and dividing by $m_K$ gives
Eq.~\eqref{eq:gl-beta-finite-survival}.
\end{proof}

Because $\rho^k=e^{-rk}$ for $r=-\log\rho>0$,
Eq.~\eqref{eq:gl-beta-survival} is an exact nonnegative exponential-mixture
representation of the GL survival sequence. Equation~\eqref{eq:gl-beta-coefficients}
uses exactly the normalized finite-horizon atom of Section~\ref{sec:methods}.
Thus the truncated GL kernel belongs to the continuous-scale version of the
mixture family, with total lag weight $m_K$ satisfying the constraint of
Eq.~\eqref{eq:mass-cap}. The endpoint cases recover the identity transformation
at $d=0$ and first differencing at $d=1$ in the corresponding limiting sense.

\paragraph{A numerical check of the fixed-dictionary discretization.}
The exact representation above permits any $\rho\in(0,1)$, whereas the
deployed configuration restricts the mixing measure to the nine decay scales of
Appendix~\ref{ap:implementation_details}. We project the truncated GL
coefficient magnitudes $(|h_k^{\mathrm{GL}}(d)|)_{k=1}^{K}$ onto the nine
exponential profiles by nonnegative least squares over the feasible set
$\{w\geq0:\sum_{j=1}^{J} w_j\leq1\}$. The relative $\ell_2$ reconstruction error
grows from $0.10\%$ at $d=0.1$ to $1.4\%$ at $d=0.6$, and reaches $9.0\%$ at
$d=0.9$; at this order the truncated GL weight approaches one and the
restricted-grid optimum reaches the constraint $\sum_{j=1}^{J} w_j\leq1$.
Away from that constraint, the fitted total lag
weight differs from the truncated GL weight $m_K$ by less than
$5\times10^{-4}$. The chosen dictionary therefore provides a low-dimensional
approximation to the exact continuous mixture, with accuracy deteriorating near
$d=1$. The estimation program learns its weights from the data by
Eq.~\eqref{eq:socp-objective}.

\section{Properties of the Estimation Program}
\label{ap:program_properties}

This section records the properties of the program in
Eqs.~\eqref{eq:socp-objective}--\eqref{eq:mass-cap} that the paper relies on:
the zero-frequency gain, the increment-variance identity, the second-order
cone reformulation, and the conditions for stable inversion and finite-horizon
reintegration.

\paragraph{Total lag weight and zero-frequency gain.}
The total lag weight controls how strongly the filter attenuates a constant
input. Since $h_0=1$ and the lag coefficients sum to
$-\sum_{j=1}^{J}w_j$, the zero-frequency gain is
\[
  H(1)=\sum_{k=0}^{K}h_k(w)=1-\sum_{j=1}^{J}w_j.
\]
The constraints $w\geq0$ and $\sum_{j=1}^{J}w_j\leq1$ therefore keep this
gain between zero and one. A total lag weight of zero gives the identity
filter; a total lag weight of one removes constant components. The
optimization learns both the total lag weight and its distribution across
the candidate decay scales.

\paragraph{Increment-variance identity.}
For a weakly stationary output process $\{Y_t\}$ with
autocovariance function $\gamma_Y(\ell)$,
\[
  \mathbb E[\widehat g_{i,\ell}(w)]
  =\mathbb E[(Y_{t+\ell}-Y_t)^2]
  =2\{\gamma_Y(0)-\gamma_Y(\ell)\},
\]
which is the lag-$\ell$ increment variance, or twice the semivariogram.

\paragraph{SOCP reformulation and convexity.}
\label{ap:socp_convexity}
The implementation expresses the maxima in
Eq.~\eqref{eq:second-order-regularizer} using envelope variables
$\kappa_\ell$ and nonnegative slacks $\eta_{i,\ell}$. This is the form the
program optimizes and the basis of the convexity argument below. Let
$q_{i,0}=\widehat v_i$ and $q_{i,\ell}=\widehat g_{i,\ell}$ for
$\ell\in\mathcal L$, and let $\nu_0=1$ and $\nu_\ell=1/|\mathcal L|$. The
program adds
\begin{equation}
  \lambda\sum_{\ell\in\mathcal L\cup\{0\}}\nu_\ell
  \left(\kappa_\ell+\sum_{i=1}^{I}\eta_{i,\ell}\right)
  \label{eq:second-order-epigraph-objective}
\end{equation}
to the objective, subject to
\begin{equation}
  q_{i,\ell}(w)\leq\kappa_\ell+\eta_{i,\ell},
  \qquad
  \kappa_\ell\geq0,\quad\eta_{i,\ell}\geq0.
  \label{eq:socp-variogram-constraint}
\end{equation}
For any fixed nonnegative local energies,
\[
  \min_{\kappa_\ell,\eta_{1,\ell},\ldots,\eta_{I,\ell}\geq0}
  \left\{\kappa_\ell+\sum_i\eta_{i,\ell}:
  q_{i,\ell}\leq\kappa_\ell+\eta_{i,\ell}\right\}
  =\max_i q_{i,\ell}.
\]
Thus the envelope--slack form is exactly the worst-window regularizer of the
main text. Its strength relative to the fidelity term is controlled by
$\lambda$. The reformulated program is
always feasible: at any $w$ in the feasible set, the nonnegative slacks can be
raised until every constraint in Eq.~\eqref{eq:socp-variogram-constraint}
holds. Its objective is a sum of nonnegative terms, so it is bounded below.
The following proposition justifies the claim in Section~\ref{sec:methods}
that each calibration solves a single SOCP.

\begin{proposition}
The optimization problem in Eqs.~\eqref{eq:socp-objective}--\eqref{eq:mass-cap}
is second-order cone representable.
\end{proposition}

\begin{proof}
The filter coefficients $h_k(w)$ are affine in $w$. For a fixed input
segment, each filtered output $\sum_{k=0}^{K}h_k(w)x_{t-k}$ is therefore
affine in $w$. The fidelity term is a sum of squared affine residuals, and
each partial sum $\sum_{u=1}^{t}y_u(w)$ in
Eq.~\eqref{eq:kpss-regularizer} is also affine in $w$. Thus both the fidelity
term and the slow-drift penalty are convex quadratics and admit standard
second-order-cone epigraph representations. The remaining objective terms in
Eq.~\eqref{eq:second-order-epigraph-objective} are linear in $(\kappa,\eta)$,
while the constraint $\sum_{j=1}^{J} w_j\leq1$ in Eq.~\eqref{eq:mass-cap} and
the nonnegativity constraints $w\geq0$, $\kappa\geq0$, and $\eta\geq0$ are
linear.

For each window $i$, the mean
$\bar y_i(w)=m^{-1}\sum_{t=1}^{m}y_{i,t}(w)$ is affine in $w$. Define
\[
  r_{i,\ell,t}(w)=
  \begin{cases}
    y_{i,t}(w)-\bar y_i(w), & \ell=0, \\
    y_{i,t+\ell}(w)-y_{i,t}(w), & \ell\in\mathcal L,
  \end{cases}
\]
for $t=1,\ldots,c_\ell$, where $c_0=m$ and $c_\ell=m-\ell$ for
$\ell\in\mathcal L$. Each residual is affine in $w$, so the local energies
\[
  q_{i,\ell}(w)
  =\frac{1}{c_\ell}\sum_{t=1}^{c_\ell}r_{i,\ell,t}(w)^2
\]
are convex quadratics. For the cone representation, collect the residuals
into $r_{i,\ell}(w)=(r_{i,\ell,1}(w),\ldots,r_{i,\ell,c_\ell}(w))$.
The constraints of Eq.~\eqref{eq:socp-variogram-constraint} are equivalently
\[
\|r_{i,\ell}(w)\|_2^2
\le
c_\ell(\kappa_\ell+\eta_{i,\ell}).
\]
The residual vector is affine in $w$, and the right-hand side is affine and
nonnegative. These constraints are therefore second-order cone representable
by a rotated cone \citep{lobo1998socp}. Consequently, every
constraint is either linear or second-order cone representable, and the
optimization problem is an SOCP after the standard epigraph reformulation of
the quadratic objective.
\end{proof}

\paragraph{Minimum phase and finite-horizon inversion.}
The downstream forecasting task maps predictions from the filtered
representation back to the original scale, so the location of the filter
zeros determines the stability of this inversion. Because \(h_0=1\), \(h_k\leq0\), and
\[
  \sum_{k=1}^{K}|h_k|
  =\sum_{j=1}^{J}w_j
  \leq1,
\]
the transfer function \(H(z)=\sum_{k=0}^{K}h_kz^{-k}\) satisfies
\[
  |H(z)|
  \geq1-\sum_{k=1}^{K}|h_k||z|^{-k}
  \geq1-\sum_{j=1}^{J}w_j,
  \qquad |z|\geq1.
\]
If \(\sum_jw_j<1\), this bound is strictly positive, so all zeros lie
strictly inside the unit circle and the filter has a stable causal inverse.

When \(H(1)=1-\sum_jw_j>0\), the fixed kernel retains a unit root in
the input. For a random walk \(x_t=x_{t-1}+\varepsilon_t\) with
stationary innovations, its output equals \(H(1)x_t\) plus a finite linear
combination of the innovations, and therefore retains a random-walk component.
The fixed-width \texttt{FD(0.5)} implementation also has a positive kernel
sum and retains a random-walk component. In both cases, the output remains
nonstationary despite any reduction in finite-window persistence.

If \(\sum_jw_j=1\), then \(H(1)=0\). No zero can lie outside the unit
circle because
\[
  \sum_{k=1}^{K}|h_k||z|^{-k}<1,
  \qquad |z|>1.
\]
Moreover, the exponential dictionary gives \(|h_k|>0\) at every lag.
Equality on the unit circle therefore requires \(z^{-k}=1\) for every
\(k\), and hence \(z=1\). Finally,
\[
  H'(1)=\sum_{k=1}^{K}k|h_k|>0,
\]
so the zero at \(z=1\) has multiplicity one. The boundary filter is not
minimum phase, and its inverse has the same unit-root instability as the
inverse of first differencing.

The forecasting implementation permits this boundary case for
finite-horizon reintegration initialized from the observed history. A
fitted kernel is rejected when its largest zero modulus exceeds
\(1+10^{-4}\). Truncated fractional differencing with \(d<1\) is strictly
minimum phase, whereas first differencing and the EWMA residual lie on the
unit-root boundary.

The boundary was rare among the fitted StaFIR filters. Using the
solver-tolerance mass-cap diagnostic, none of the 469 post-selection
real-data fits reached it; the maximum total lag mass was \(0.993\) for
log prices and \(0.981\) for realized volatility. The cap was active for
139 of 25,920 post-selection synthetic fits (\(0.54\%\)), spread across
multiple values of the true memory parameter.

\section{Scores and Selection}
\label{ap:scores_selection}

\subsection{Evaluation scores and effect turnover}
\label{ap:score_definitions}

The stationarity score of Section~\ref{sec:scores} is computed on the
transformed evaluation segment $y$, whereas the similarity score compares $y$
with the aligned raw segment $x$. The same definitions score the validation
outputs during selection.

\paragraph{Stationarity score.}
Let $\{\widehat\mu_t\}$ and $\{\widehat v_t\}$ be the means and sample
variances over sliding windows of length $m=128$, and for each
$\ell\in\mathcal L=\{1,2,4,8,16,32\}$ let $\{\widehat g_{t,\ell}\}$ be the
corresponding sliding-window variograms, the mean squared lag-$\ell$
increments within each window. Let $\sigma_{\mathrm{samp}}(\cdot)$ denote the
sample standard deviation over the values indexed by $t$. Define the defects
\[
D_\mu=\frac{\sigma_{\mathrm{samp}}(\widehat\mu_t)}{\sigma_{\mathrm{samp}}(y_t)},
\qquad
D_v=\sigma_{\mathrm{samp}}\!\left\{\log\widehat v_t\right\},
\qquad
D_g=\max_{\ell\in\mathcal L}
\frac{\sigma_{\mathrm{samp}}(\widehat g_{t,\ell})}{\widehat g_{\ell}},
\]
where $\widehat g_{\ell}=(n-\ell)^{-1}\sum_{t=1}^{n-\ell}(y_{t+\ell}-y_t)^2$
is the full-segment variogram. Degenerate denominators are floored at machine
precision; the full guard rules are in the code. Each defect maps
to a bounded component score $100/(1+D)$, and the aggregate is the worst
component,
\[
\Sstat(y)
=\min\left\{\frac{100}{1+D_\mu},\frac{100}{1+D_v},\frac{100}{1+D_g}\right\}.
\]
The three components are reported separately as the mean, log-variance, and
variogram stability in Table~\ref{tab:real-log-price-diagnostics}.

\paragraph{Similarity score.}
Let $\rho_{\mathrm P}(x,y)$ be the full-sample Pearson correlation,
$\rho_{\mathrm P,\mathcal T}(x,y)$ the Pearson correlation restricted to the
tail set
\[
  \mathcal T=\{t:x_t\le Q_{0.10}(x)\}\cup\{t:x_t\ge Q_{0.90}(x)\},
\]
with quantiles computed on the raw segment, and let
$\operatorname{dCor}(x,y)\in[0,1]$ be the distance correlation. Define
\[
a_1(x,y)=\frac{\rho_{\mathrm P}(x,y)+1}{2},
\qquad
a_2(x,y)=\frac{\rho_{\mathrm P,\mathcal T}(x,y)+1}{2},
\qquad
a_3(x,y)=\operatorname{dCor}(x,y),
\]
and let $\mathcal I(x,y)=\{r\in\{1,2,3\}:a_r(x,y)\text{ is defined}\}$ index
the available components. The similarity score averages the available
components,
\[
S_{\mathrm{sim}}(x,y)
=\frac{100}{|\mathcal I(x,y)|}
\sum_{r\in\mathcal I(x,y)}a_r(x,y).
\]
A Pearson component is omitted when it is undefined because either input has
zero empirical variance. Distance correlation is set to zero when either
distance variance vanishes. Before either score is evaluated, non-finite
observations are removed pairwise. Evaluation fails when fewer than 20 finite
aligned observations remain.

\paragraph{Effect turnover.}
In the rolling experiments we track how much the realized transformation
changes between consecutive origins. The \emph{effect turnover} at origin $j$
applies the current and the preceding fitted filters to the same evaluation
segment,
\begin{equation}
T_y(j) = \frac{\|F_j x_{\mathcal E_j}-F_{j-1}x_{\mathcal E_j}\|_2}{\|F_{j-1}x_{\mathcal E_j}\|_2+\varepsilon_{\mathrm{den}}},
\label{eq:effect-turnover}
\end{equation}
where $F_j$ is the filtering operator fitted at origin $j$, $\mathcal E_j$ is
the evaluation segment at origin $j$, and $\varepsilon_{\mathrm{den}}=10^{-12}$
stabilizes the denominator. Turnover is reported for the adaptive methods:
the two adaptive fractional-differencing rules and \texttt{StaFIR-context}.
Fixed transformations, including \texttt{FD(0.5)}, attain $T_y=0$ by
construction and are marked as adaptation-free.

\subsection{Contextual selection rule}
\label{ap:empirical-rule}

Let $x^{\mathrm{val}}$ be the standardized raw validation segment and
$\Delta x^{\mathrm{val}}$ its first difference, with the final training
observation supplying the history of the first difference. The raw and
first-differenced validation outputs define a contextual reference level for the empirical stationarity score
\[
S_{\min}
=
\Sstat(x^{\mathrm{val}})
+
\zeta
\left[
\Sstat(\Delta x^{\mathrm{val}})
-
\Sstat(x^{\mathrm{val}})
\right]_+,
\qquad [u]_+=\max(u,0),
\]
with the target fraction fixed at $\zeta=0.9$ in the main experiments. Both
anchors are re-estimated on every window. The endpoints are interpretable:
$\zeta=0$ sets the target at the raw-series score, whereas $\zeta=1$ matches
the score attained by first differencing whenever that score is higher.

Because the target and all candidate scores are estimated from a short
validation segment, selection treats $S_{\min}$ as a soft target.
Let $\mathcal C$ denote the candidate set of the method
and let $c^\star_{j-1}$ be the candidate deployed at the preceding rolling
origin. At origin $j$, the selected candidate satisfies
\[
c^\star_j
\in
\arg\max_{c\in\mathcal C}\;
S_{\mathrm{sim}}(x^{\mathrm{val}},y_c^{\mathrm{val}})
-
\gamma_{\mathrm{sel}}
\left[S_{\min}-\Sstat(y_c^{\mathrm{val}})\right]_+^2
-
c_{\mathrm{sw}}\,\mathbf 1\!\left[c\ne c^\star_{j-1}\right].
\]
The experiments use the shortfall curvature $\gamma_{\mathrm{sel}}=0.4$ and
the switching cost $c_{\mathrm{sw}}=10$. The second term penalizes only a
stationarity shortfall and grows quadratically below the target. The third
term discourages changes from the incumbent selection. It is omitted at the
first rolling origin and in the fixed synthetic grid, where selection occurs
once per scenario. The scheduled-change study rolls selection forward and
applies the switching cost at every later origin. The criterion is defined for
every evaluable candidate. Whether the selected candidate reaches $S_{\min}$,
and its shortfall from the target, are
recorded as diagnostics. The same rule is applied to the
$(\lambda_s,\lambda)$ grid of \texttt{StaFIR-context} and to the
fractional-order grid of \texttt{FD-context}. The sensitivity of the rolling
log-price results to $\zeta$, $\gamma_{\mathrm{sel}}$, and $c_{\mathrm{sw}}$
is reported in Appendix~\ref{ap:selection-sensitivity}.

\section{Experimental Details}
\label{ap:experimental_details}

\subsection{Compared transformations}
\label{ap:baseline_details}

Table~\ref{tab:methods-overview} summarizes the compared transformations.
\texttt{StaFIR-context} denotes the proposed filter of
Eqs.~\eqref{eq:socp-objective}--\eqref{eq:mass-cap} selected by the
contextual rule of Appendix~\ref{ap:empirical-rule}. All FIR
transformations, that is, \texttt{EWMA}, the fractional-differencing rules,
and \texttt{StaFIR-context}, use the maximum lag $K=100$, so every filter
combines the current observation with its 100 predecessors. The
\texttt{EWMA} baseline is the residual of a fixed exponentially weighted
moving average,
\[
  y_t=x_t-\sum_{k=0}^{K}q_kx_{t-k},\qquad
  q_k=\frac{a(1-a)^k}{\sum_{r=0}^{K}a(1-a)^r},\qquad
  a=\frac{2}{K+2}.
\]
The smoothing factor $a=2/(K+2)$ is the standard span convention for a span of
$K+1$ observations, matching the filter length. The \texttt{rolling-z}
baseline centers and scales each observation by the mean and standard
deviation of its trailing $K+1$ observations, including the current one.

\begin{table}[t]
  \caption{Compared preprocessing methods.}
  \label{tab:methods-overview}
  \centering
  \scriptsize
  \setlength{\tabcolsep}{4pt}
  \begin{tabular}{@{}lcll@{}}
    \toprule
    Method & Adaptive & Parameter & Selection rule \\
    \midrule
    \texttt{raw}, \texttt{1-diff} & no & -- & fixed references \\
    \texttt{EWMA}, \texttt{rolling-z} & no & window $K+1=101$ & fixed \\
    \texttt{FD(0.5)} & no & order $d=0.5$ & fixed a priori \\
    \texttt{FD-adf} & yes & order $d$ & smallest order with ADF rejection \\
    \texttt{FD-context} & yes & order $d$ & contextual rule of Appendix~\ref{ap:empirical-rule} \\
    \texttt{StaFIR-context} & yes & $(\lambda_s,\lambda)$ & contextual rule of Appendix~\ref{ap:empirical-rule} \\
    \bottomrule
  \end{tabular}
\end{table}

\subsection{Rolling evaluation protocol}
\label{ap:rolling-protocol}

Each fit-and-evaluate step uses a 2{,}000-observation window split into
1{,}280 training, 320 validation, and 400 evaluation observations
(Figure~\ref{fig:rolling-design}). Hyperparameters governed by the contextual
rule are selected on the validation segment, while \texttt{FD-adf} selects its
order on the calibration data. The \texttt{StaFIR-context} weights are then
re-fitted on the complete 1{,}600-observation calibration period, and metrics
are computed only on the evaluation segment. The fixed synthetic grid applies
this split once per scenario. The scheduled-change paths roll it forward by
200 observations and the market series by 63 observations. In every rolling
experiment, all methods at an origin receive the same standardized calibration
data and are evaluated on the same held-out segment.

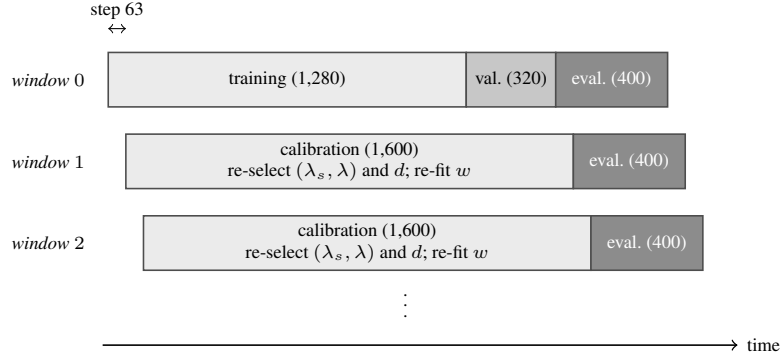
\begin{figure}[h]
  \centering
  \begin{tikzpicture}[x=0.37cm, y=0.90cm,
      seg/.style={draw=black!70, semithick},
      lab/.style={font=\scriptsize},
      rowlab/.style={font=\scriptsize\itshape, anchor=east}]
    \node[rowlab] at (-0.5,0.4) {window $0$};
    \draw[seg, fill=black!8]  (0,0)    rectangle (12.8,0.8);
    \draw[seg, fill=black!22] (12.8,0) rectangle (16,0.8);
    \draw[seg, fill=black!45] (16,0)   rectangle (20,0.8);
    \node[lab] at (6.4,0.4)  {training (1{,}280)};
    \node[lab] at (14.4,0.4) {val.\ (320)};
    \node[lab, text=white] at (18,0.4) {eval.\ (400)};
    \node[rowlab] at (-0.5,-0.8) {window $1$};
    \draw[seg, fill=black!8]  (0.63,-1.2)  rectangle (16.63,-0.4);
    \draw[seg, fill=black!45] (16.63,-1.2) rectangle (20.63,-0.4);
    \node[lab, align=center] at (8.63,-0.8) {calibration (1{,}600)\\re-select $(\lambda_s,\lambda)$ and $d$; re-fit $w$};
    \node[lab, text=white] at (18.63,-0.8) {eval.\ (400)};
    \node[rowlab] at (-0.5,-2.0) {window $2$};
    \draw[seg, fill=black!8]  (1.26,-2.4)  rectangle (17.26,-1.6);
    \draw[seg, fill=black!45] (17.26,-2.4) rectangle (21.26,-1.6);
    \node[lab, align=center] at (9.26,-2.0) {calibration (1{,}600)\\re-select $(\lambda_s,\lambda)$ and $d$; re-fit $w$};
    \node[lab, text=white] at (19.26,-2.0) {eval.\ (400)};
    \node[lab] at (10.63,-2.8) {$\vdots$};
    \draw[<->, thin] (0,1.15) -- (0.63,1.15) node[midway, above, lab] {step 63};
    \draw[->, semithick] (-0.2,-3.5) -- (22.5,-3.5) node[right, lab] {time};
  \end{tikzpicture}
  \caption{Rolling evaluation design for the market series. Each
  2{,}000-observation window consists of training (1{,}280), validation (320),
  and evaluation (400) segments, with a step of 63 observations between
  recalibrations. Every later window re-selects the penalty pair and the
  adaptive fractional order and re-fits the weight vector $w$ on its own
  calibration segment. The fixed \texttt{FD(0.5)} baseline is unchanged.}
  \label{fig:rolling-design}
\end{figure}

\paragraph{Standardization.}
At each origin, the mean and scale are re-estimated from that origin's training
segment and the resulting affine transformation is shared by every method.

\paragraph{Selection period.}
Let $M$ denote the number of rolling origins between two hyperparameter
selections. The main experiments use $M=1$, so the contextual rule selects at
every origin. For $M>1$, selection runs at every $M$th origin. At the origins
in between, the method keeps the previously selected hyperparameters and only
re-fits the standardization and the filter weights on the new calibration
window. We call this a \emph{carried re-fit}. The endpoint $M=\infty$ keeps
the selection made at the first origin for the whole sample. The selection
period is varied in Appendix~\ref{ap:adaptation-cost}.

\paragraph{Fixed-order reference path.}
At every origin of the market series, every fractional order
$d\in\{0,0.01,\ldots,1\}$ is additionally evaluated on the held-out segment.
These evaluations trace the fixed-order fractional-differencing path shown in
Figure~\ref{fig:real-levels-tradeoff}.

\subsection{Candidate grids and solver}
\label{ap:implementation_details}

\paragraph{StaFIR candidates.}
The selected pair consists of the low-frequency penalty $\lambda_s$ and the
local second-order penalty $\lambda$ of Eq.~\eqref{eq:socp-objective},
\[
\lambda_s\in\{0\}\cup\{10^{-4+q/2}:q=0,\ldots,12\},
\qquad
\lambda\in\{0.1,0.2,0.5,1,2,5,10\},
\]
giving $14\times7=98$ candidates. These pairs form the candidate set
$\mathcal C$ of the contextual rule in Appendix~\ref{ap:empirical-rule}. Each
candidate is estimated on the training segment, and its validation output is
scored by that rule. The fixed decay scales are
$\tau_j\in\{0.5,1,2,4,8,16,32,64,128\}$, so $J=9$. Every candidate program
receives all nine scales, and the nonnegative weights decide which of them are
active. Unused scales receive $w_j=0$. The stationarity penalties use local
windows of length $m=128$ and lags $\mathcal L=\{1,2,4,8,16,32\}$.

\paragraph{Fractional orders.}
All adaptive fractional-differencing rules search the grid
$d\in\{0,0.01,\ldots,1\}$; the fixed \texttt{FD(0.5)} baseline stays at
$d=0.5$. The \texttt{FD-adf} baseline selects the smallest order whose
transformed calibration output rejects the ADF unit-root null at 5\%. For
\texttt{FD-context}, this grid is the candidate set of the contextual rule,
applied with the same training and validation split as
\texttt{StaFIR-context}.

\paragraph{Solver.}
The SOCP is solved with CLARABEL, with at most 50 iterations and a tolerance of
$10^{-6}$ on the absolute and relative duality gaps and on feasibility. The
program is always feasible and bounded below
(Appendix~\ref{ap:program_properties}). A training fit on 1{,}280
observations produces 1{,}180 aligned outputs and $I=9$ windows, so it carries
9 weight variables, 7 envelope variables, and 63 slack variables under 63
rotated second-order-cone constraints, together with the weight-budget and
nonnegativity constraints. The calibration re-fit on 1{,}600 observations
gives $I=11$ windows and is slightly larger. The implementation applies a thin
QR factorization to the coefficients and constant terms of the affine
residuals defined in Appendix~\ref{ap:socp_convexity}. Orthogonal invariance
preserves each sum of squared residuals, so each quadratic block can be
evaluated using at most $J+1$ residual components instead of one per
observation.

An infeasible status triggers one retry with the series rescaled by its maximum
absolute value, more iterations, and tighter tolerances; rescaling multiplies
the objective by a constant and leaves the minimizer unchanged. The compiled
DPP problem is reused across the grid. Each numerical solve starts cold,
making a candidate's reported fit independent, up to the
stated solver tolerance, of preceding candidates, grid traversal order, and
parallel-worker chunking. The reported timings use this configuration.
Solver status, retry use, and selected shortfall are recorded.

\paragraph{Stationarity tests.}
Table~\ref{tab:appendix_implementation_details} records the packages,
routines, and settings used for the ADF, PP, and KPSS tests and for the SOCP
solver. All tests use a constant term only. The ADF test uses one fixed lag,
while the PP bandwidth and the KPSS lag length are selected automatically.
\input{tables/experiments}

\subsection{Synthetic data}
\label{ap:synthetic_data_generation}

The ARFIMA model extends the classical ARIMA representation by allowing a
real fractional-integration order $d_{\mathrm{true}}$
\citep{grangerFracDiff1980,hosking1981fractional,baillieLongMem1996}. Under
the usual causality and invertibility conditions, an ARFIMA process is
covariance-stationary for $-1/2<d_{\mathrm{true}}<1/2$, while
$d_{\mathrm{true}}\ge 1/2$ corresponds to nonstationary persistence, the
regime targeted by the main synthetic grid below. A GARCH innovation process
adds conditional heteroscedasticity in the form of volatility clustering
\citep{bollerslev1986garch}.

For the GARCH innovations, we set $a = 0.1p_{\mathrm{garch}}$ and
$b = 0.9p_{\mathrm{garch}}$, so that $a + b < 1$ for all scenarios. The
conditional variance recursion is
\begin{align*}
  \varepsilon_t &= \sigma_t z_t, \qquad z_t\overset{\text{i.i.d.}}{\sim} \sqrt{6/8}\,t_8, \\
  \sigma_t^2 &= \omega + a\varepsilon_{t-1}^2 + b\sigma_{t-1}^2.
\end{align*}
Here, $\omega = 1-p_{\mathrm{garch}}$, and the innovations $z_t$ follow a
Student-$t$ distribution with 8 degrees of freedom, rescaled to unit variance.
The initial unconditional GARCH variance is
$\sigma_0^2 = \omega/(1 - a - b) = 1$ by construction. Because the model is
$\mathrm{ARFIMA}(0,d_{\mathrm{true}},0)$, the short-memory signal is
$u_t = \varepsilon_t$, and long memory is introduced by expanding fractional
integration,
\[
  x_t = (1 - B)^{-d_{\mathrm{true}}}u_t \approx \sum_{k=0}^{t}h_k^{\mathrm{GL}}(-d_{\mathrm{true}})u_{t-k},
\]
where $h_k^{\mathrm{GL}}(-d_{\mathrm{true}})$ is given by the recursion in
Eq.~\eqref{eq:fracdiff_rec} with $-d_{\mathrm{true}}$ in place of $d$. The
integration history therefore grows with $t$, and each retained observation
uses all available simulated presample history. This expanding rule is used
only to generate the data. The evaluated transformations remain fixed-width
filters with the maximum lag of Appendix~\ref{ap:baseline_details}.

\paragraph{Main grid.}
The synthetic scenarios are defined on the Cartesian grid
\[
  d_{\mathrm{true}} \in [0.5, 1.0], \qquad p_{\mathrm{garch}} \in [0.5, 0.99],
\]
with 10 equally spaced values along each axis, including the endpoints,
yielding $10 \times 10 = 100$ scenarios per seed. We use 30 random seeds,
$s\in\{42,43,\ldots,71\}$, for a total of 3{,}000 scenarios. For each
scenario, $N_{\mathrm{obs}} = 2{,}000$ observations are retained after
5{,}000 discarded burn-in observations, which reduces the influence of the
initial GARCH variance and of the initialization of the fractional-integration
history. Each scenario initializes NumPy's \texttt{SeedSequence} with the seed
and its two grid indices, so results are invariant to the scenario iteration
order.

\paragraph{Stationary companion grid.}
To measure unnecessary transformation, we repeat the experiment for
$d_{\mathrm{true}}\in\{0,0.01,\ldots,0.49\}$ and the same 10 equally spaced
persistence values $p_{\mathrm{garch}}\in[0.5,0.99]$ as in the main grid,
giving $50\times10=500$ scenarios per seed and 15{,}000 scenarios in total.
These processes lie below the ARFIMA stationarity boundary. Seeds, sample
length, split, and candidate grids are those of the main grid. The two right-hand panels
of Figure~\ref{fig:synthetic-adaptivity} additionally vary
$p_{\mathrm{garch}}\in[0,0.49]$ on a $20\times10$ grid over
$d_{\mathrm{true}}\in[0,1]$ with the same seeds.

\paragraph{Regime shift scenarios.}
The scheduled-change protocol imitates a regime shift by replacing the grid 
with four two-regime designs. We consider changes of $d_{\mathrm{true}}$ and 
$p_{\mathrm{garch}}$ one at a time: when one parameter moves between 
0.50 and 0.95 in each direction, the other one is held at 0.75. Each design
is one continuous path of $N_{\mathrm{obs}}=5{,}000$ retained observations after 5{,}000
discarded burn-in observations, with the switch at observation 2{,}600. Within
a design family, the forward and reverse paths consume the same standardized
shock stream, so the two directions differ only in the parameter schedule. The
rolling protocol of Appendix~\ref{ap:rolling-protocol} is then applied along
the path with a step of 200 observations, and the standardization,
hyperparameters, and filter weights are re-estimated at every origin. We use
the same 30 seeds. An origin is \emph{pre} when neither its
calibration nor its evaluation segment contains a post-switch observation,
\emph{post} when both consist entirely of post-switch observations, and
\emph{transition} otherwise.

\subsection{Financial series}
\label{ap:real_data_details}

Table~\ref{tab:real-data} lists the four daily log-price series. Origins are
the recalibration dates of the rolling protocol of
Appendix~\ref{ap:rolling-protocol}, and the origin span is the period they
cover. Each series retains its own observation calendar. The Bitcoin index is
observed on a calendar-day grid rather than a trading-day grid. For the other
series, we retain the open-market days supplied by Bloomberg and do not insert
or interpolate holidays.

\begin{table}[t]
  \caption{Daily log-price series after cleaning.}
  \label{tab:real-data}
  \centering
  \scriptsize
  \setlength{\tabcolsep}{4pt}
  \begin{tabular}{@{}llrlrl@{}}
    \toprule
    Series & Underlying & Obs. & Sample span & Origins & Origin span \\
    \midrule
    \texttt{sp500} & S\&P 500 index (\texttt{SPX Index}) & 6{,}539 & 2000-01-03 -- 2025-12-31 & 73 & 2006-05-15 -- 2024-05-22 \\
    \texttt{eurusd} & EUR/USD rate (\texttt{EURUSD Curncy}) & 6{,}783 & 2000-01-03 -- 2025-12-31 & 76 & 2006-02-17 -- 2024-03-29 \\
    \texttt{cds\_idx\_sp} & On-the-run Crossover five-year index spread & 5{,}307 & 2005-08-30 -- 2025-12-31 & 53 & 2011-10-17 -- 2024-05-07 \\
    \texttt{xbx\_index} & CoinDesk Bitcoin Price Index (\texttt{XBX}) & 4{,}077 & 2014-11-03 -- 2025-12-31 & 33 & 2019-03-21 -- 2024-09-26 \\
    \bottomrule
  \end{tabular}
\end{table}

\subsection{Realized-volatility forecasting task}
\label{ap:downstream-task}

\paragraph{Task and pipeline.}
For the forecasting task, define a positive causal realized-variance target
from the trailing five daily returns, and let its log level be
\[
  V_t=\sum_{r=0}^{4}(x_{t-r}-x_{t-r-1})^2>0,
  \qquad \ell_t=\log V_t.
\]
The input series is $\ell_t$. A window with $V_t=0$ is discarded upstream
rather than floored. The four inputs therefore describe the same markets as
the rolling log-price study but form a distinct realized-volatility data set.
The transformation pipeline, calibration windows, selection rule, and fitted
filters are otherwise unchanged. Constructing $V_t$ removes the first five
observations, leaving 234 realized-volatility origins rather than the 235
log-price origins. Only the S\&P~500 count changes, from 73 to 72. At each
origin, a ten-lag ridge autoregression uses the transformed $\ell_t$ values
to make a one-step forecast. The five-day window defines the target, not the
autoregressive order. The model is iterated recursively to
horizons $h\in\{1,5,20,40,80\}$, and the method's fitted causal FIR filter is
inverted over the forecast path to return the prediction
$\widehat\ell_{t+h\mid t}$ to the level of the raw $\ell_t$. Reintegration
uses the observed history for initialization and permits the unit-root
boundary case described in Appendix~\ref{ap:program_properties}. Every method
therefore predicts the same future log realized variance. Only the
representation changes.

\paragraph{Forecasting model and recursion.}
Let $u$ denote a method's transformed standardized log realized variance. For
an origin $o$, calibration start $c=o-C+1$, FIR warm-up of $K$ observations,
and $L=10$ lags, the one-step model uses exactly the rows
$t=\max(c,K)+L-1,\ldots,o-1$, with features $(u_t,\ldots,u_{t-L+1})$ and
target $u_{t+1}$. With the 1{,}280/320 training--validation split, tuning rows
satisfy $t+1\leq c+1279$ and validation rows satisfy $t\geq c+1280$; the
boundary row is unused. For standardized feature vector $z_i$ and response
$y_i=u_{t+1}$, the model solves
\[
 \min_{a,\beta}\;n^{-1}\sum_i(y_i-a-z_i^\top\beta)^2+
 \lambda_{\mathrm{ridge}}\lVert\beta\rVert_2^2,
\]
so the intercept $a$ is unpenalized. The ridge penalty is selected from
$\{10^{-4},10^{-3},10^{-2},10^{-1},1,10,100\}$ by validation MSE, and the
model is then refit on all usable calibration rows.
For learned filters, ridge tuning uses training and validation
representations constructed with the filter refitted on the full
calibration window. Thus, internal validation is not fully held out
from preprocessing estimation, while the reported forecasting
losses are evaluated on subsequent held-out dates.
Recursion starts from the
observed transformed lags at $t$, and predicted transformed values replace
future lags. The FIR inversion is seeded with the observed standardized log
realized variance through $t$. All FIR representations use these same
settings and target dates.

\paragraph{Scoring and uncertainty.}
Scoring is continuous and uses each date once. The positive forecast implied
by the reintegrated log prediction is
$\widehat V_{t+h\mid t}=\exp(\widehat\ell_{t+h\mid t})$, and the loss is
\[
 \operatorname{QLIKE}(V,\widehat V)=\frac{V}{\widehat V}
 -\log\!\left(\frac{V}{\widehat V}\right)-1
 =\exp(\ell-\widehat\ell)-(\ell-\widehat\ell)-1.
\]
The implementation evaluates the final log-domain expression without
materializing the exponentiated forecast. It applies no log-to-level bias
correction, prediction clipping, or floor. Finite extremes are scored, and any
nonfinite prediction or loss aborts the replay with its
series--method--horizon--date context. The main comparison reports QLIKE
relative to \texttt{raw}. Uncertainty uses 10{,}000 paired circular
stationary-bootstrap resamples of the chronological date-level losses within
each series, with expected block length 20 dates. Method and raw losses share
every resample, the raw-loss denominator is recomputed in every draw, and the
four per-series relative effects are averaged with equal weight.
Series are resampled independently, so the panel intervals account
for serial dependence within each series but do not capture
contemporaneous dependence across markets.

\paragraph{Methods included.}
Because recursive reintegration requires a fitted linear filter,
\texttt{rolling-z} is not part of this task. The ADF rule selects $d=0$ at
every realized-volatility origin and therefore coincides with \texttt{raw}.
The \texttt{FD(0.5)} baseline uses the order proposed by
\citet{hasslerPohle2023forecasting}. As an external reference, we additionally
evaluate \texttt{log-HAR}, which predicts log realized variance from its
current, weekly, and monthly averages using a separate direct model at each
forecast horizon \citep{corsi2009har}. Its comparison with the other methods
reflects changes in both the representation and the forecasting model.

\section{Additional Results}
\label{ap:additional-results}
\setcounter{topnumber}{3}
\renewcommand{\topfraction}{0.90}
\renewcommand{\textfraction}{0.05}

This section follows the order of Section~\ref{sec:experiments}. It reports
the full synthetic score table and the regime shift study, the per-series
rolling log-price diagnostics with their paired uncertainty and turnover, and
the direct forecasting comparisons. It then reports the constraint and
representation ablations, the sensitivity of the selection rule, and the
adaptation cost across selection periods.

\subsection{Synthetic grid}
\label{ap:synthetic-additional-results}

Table~\ref{tab:synthetic-results} reports the full score and test table on
the main synthetic grid behind Figure~\ref{fig:synthetic-adaptivity}.

\input{generated/tables/synthetic_results}

On the stationary companion grid, the similarity scores of both contextual
methods remain above 96 across the full range of $d_{\mathrm{true}}$. The
\texttt{FD(0.5)} score instead falls from 92.3 at $d_{\mathrm{true}}=0$ to
79.9 at $d_{\mathrm{true}}=0.49$, while first differencing falls from 81.4 to
62.5. At $d_{\mathrm{true}}=0$, first differencing also lowers the
stationarity score from 80.4 for the raw series to 79.1. These results
quantify the unnecessary transformation visible in the left half of
Figure~\ref{fig:synthetic-adaptivity}.

\paragraph{Diagnostics within the stationary range.}
Table~\ref{tab:stationary-input-diagnostics} reports diagnostics on the raw
inputs at four selected orders in the stationary companion grid. KPSS
non-rejection falls from 97.3\% at $d_{\mathrm{true}}=0$ to 37.7\% at
$d_{\mathrm{true}}=0.49$, although all four parameter settings lie below
the theoretical covariance-stationarity boundary. This is consistent with
the short-memory calibration of KPSS, which can reject covariance-stationary
fractionally integrated alternatives~\citep{leeSchmidt1996kpss}.
The mean stationarity score also falls, from 80.4 to 73.1, showing that it
remains sensitive to stationary persistence. Thus, $\Sstat$
measures finite-window empirical stability and varies even within the
covariance-stationary parameter range.

\begin{table}[htbp]
  \caption{Diagnostics on raw inputs in the stationary companion grid.
  Each row averages 30 seeds and 10 equally spaced GARCH persistence settings
  in $[0.5,0.99]$, for 300 evaluations on untransformed held-out segments of
  400 observations. KPSS uses a constant, automatic lag selection, and a
  5\% significance level. No evaluation in these rows failed or lacked a
  KPSS decision. All four fractional-integration orders lie below the
  theoretical covariance-stationarity boundary.}
  \label{tab:stationary-input-diagnostics}
  \centering
  \small
  \begin{tabular}{@{}rrr@{}}
    \toprule
    $d_{\mathrm{true}}$ & Mean $\Sstat$ & KPSS non-rejection (\%) \\
    \midrule
    0.00 & 80.4 & 97.3 \\
    0.20 & 79.1 & 71.7 \\
    0.40 & 75.3 & 47.0 \\
    0.49 & 73.1 & 37.7 \\
    \bottomrule
  \end{tabular}
\end{table}

\FloatBarrier
\subsection{Regime shift through scheduled parameter changes}
\label{ap:regime-shift-results}

The regime shift experiment of Appendix~\ref{ap:synthetic_data_generation}
contains increases and decreases in either the fractional-integration order
or the GARCH persistence. Table~\ref{tab:synthetic-regime-transition} reports
results by phase with standard errors over 30 seeds. Both contextually
selected methods adjust in the expected direction. During transition windows,
\texttt{StaFIR-context} generally pays less similarity cost but has slightly
higher mean turnover than \texttt{FD-context}.

\input{generated/tables/regime_results}

\FloatBarrier
\subsection{Rolling log-price series}
\label{ap:real-additional-results}

\texttt{StaFIR-context} and \texttt{FD-context} share the training and
validation split and the contextual selection rule and differ only in the
candidate filter family. The ADF, PP, and KPSS outcomes are held-out
diagnostics, except that \texttt{FD-adf} uses ADF on its calibration segment
to choose $d$. Reported metrics condition on a successful fit, while failure
counts retain every attempted fit.

\paragraph{Score components and tests.}
Table~\ref{tab:real-log-price-diagnostics} breaks the two aggregate scores of
Appendix~\ref{ap:score_definitions} into their three stationarity and three
similarity components and places them next to the test outcomes, pooled over
the four series. For \texttt{StaFIR-context}, ADF rejects the unit-root
null at most origins, whereas KPSS has a low non-rejection rate. ADF and PP
assess the same unit-root null, while KPSS assesses stationarity under a
short-memory calibration. This calibration makes KPSS sensitive to
stationary long memory, as illustrated by the raw-input results in
Table~\ref{tab:stationary-input-diagnostics}.

\input{generated/tables/real_log_price_diagnostics}

\paragraph{Local Whittle residual memory.}
The LW $\widehat d$ column of Table~\ref{tab:real-tests-main} applies the
local Whittle estimator of Section~\ref{sec:related_work_stat} to each
held-out output with bandwidth $m=\lfloor400^{0.65}\rfloor=49$. Estimates are
pooled across rolling origins. The reported intervals are 95\%
stationary-bootstrap intervals with 10{,}000 replications and an expected
block length of seven origins.

For \texttt{StaFIR-context}, the within-series means of the held-out local
Whittle estimates are 0.47 for the credit index, 0.46 for EUR/USD, 0.49 for
Bitcoin, and 0.67 for the S\&P~500. These are descriptive estimates across
rolling origins. The pooled estimate, $0.53\,[0.49,0.58]$, places residual
memory near the $d=0.5$ stationarity boundary, with the interval spanning both
sides.

\paragraph{Paired uncertainty.}
\label{ap:real-paired-main}
Table~\ref{tab:real-paired-comparisons} forms differences within series and
rolling origin, then applies a stationary bootstrap separately to each
chronological series. It uses 10{,}000 resamples, a fixed seed, and an
expected block length of seven origins. The same draw is reused across metrics
and comparators. The panel first averages paired means within series and then
gives the four series equal weight. At expected block lengths 13 and 25, the
\texttt{StaFIR-context}--\texttt{FD-context} stationarity-gain intervals
remain above zero, respectively $[0.50,1.60]$ and $[0.61,1.50]$. The
similarity-cost intervals, $[-2.55,0.59]$ and $[-2.34,0.35]$, both cross
zero. The $\Delta|\widehat d\,|$ column applies the same procedure to the
absolute local Whittle residual memory. The equal-weight interval,
$-0.08$ $[-0.11,-0.05]$, excludes zero on an estimator neither method
optimizes, with the S\&P~500 showing a negligible mean difference. These intervals quantify serial
uncertainty conditional on the four observed series, with marginal coverage
for each contrast.

\input{generated/tables/real_paired_comparisons}

\paragraph{Turnover frequency and magnitude.}
Effect turnover has a frequency and a magnitude.
Figure~\ref{fig:real-change-behavior} shows that \texttt{StaFIR-context}
changes its transformation materially at more origins than the fractional
rules, but its realized changes are much smaller. \texttt{FD-context} instead
makes less frequent, larger jumps in the selected order.

\begin{figure*}[t]
  \centering
  \includegraphics[width=0.95\textwidth]{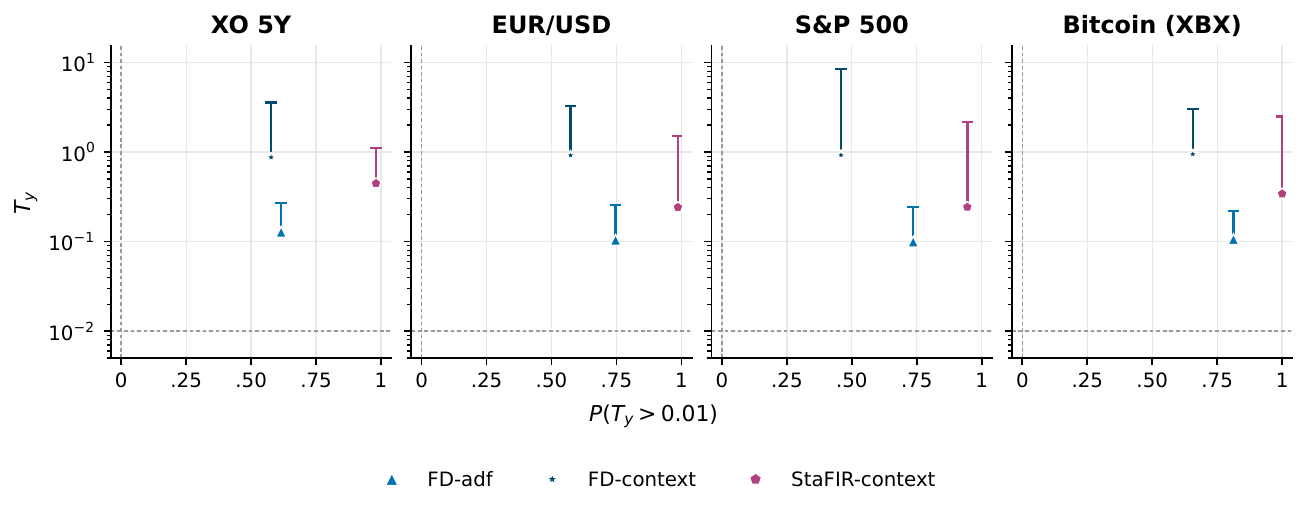}
  \caption{Frequency and magnitude of material transformation changes for the
  three adaptive methods on the log-price series. The horizontal axis,
  $P(T_y>0.01)$, is the share of consecutive origin pairs whose effect turnover
  $T_y$ exceeds 0.01. On its log scale, the
  vertical axis reports the median $T_y$ among only those origin pairs that
  exceed the threshold.
  Whiskers extend to the conditional 90th percentile.}
  \label{fig:real-change-behavior}
\end{figure*}

\subsection{Realized-volatility forecasting}
\label{ap:downstream-additional-results}

The task, its pipeline, and its uncertainty construction are defined in
Appendix~\ref{ap:downstream-task}. The comparison against \texttt{raw} is
Table~\ref{tab:downstream-results} in the main text, and
Figure~\ref{fig:downstream-horizons} shows it per series. The conclusions are
stable across bootstrap block lengths: at expected block lengths 10, 40, and
80 dates, the \texttt{StaFIR-context} intervals through $h=40$ stay below
zero, and the $h=80$ interval crosses zero at every length.

\begin{figure*}[t]
  \centering
  \includegraphics[width=0.95\textwidth]{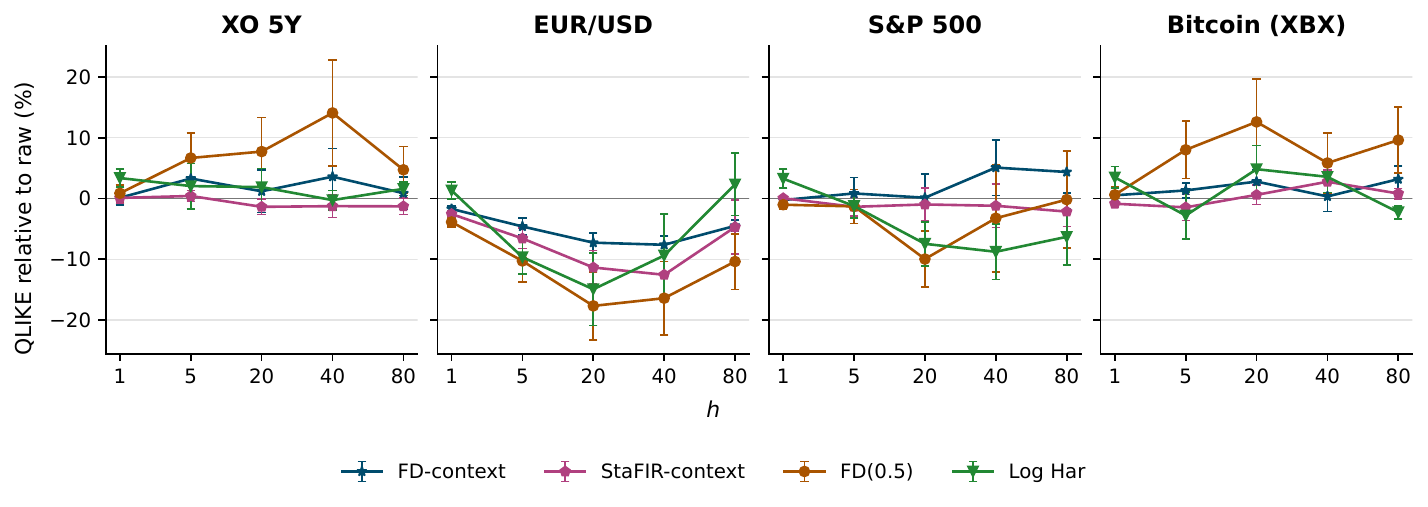}
  \caption{Per-series forecast performance on the log five-day
  realized-variance series. At each forecast horizon $h$, points report QLIKE
  relative to \texttt{raw}, in percent.
  Whiskers are one paired stationary-bootstrap standard error.
  \texttt{FD-context}, \texttt{StaFIR-context}, and \texttt{FD(0.5)} share the
  common forecasting model. \texttt{log-HAR} is an external task-specific
  reference with a different, direct forecasting model. \texttt{FD-adf}
  coincides with \texttt{raw} and is omitted.}
  \label{fig:downstream-horizons}
\end{figure*}

\paragraph{Pairwise comparisons with reintegration.}
Table~\ref{tab:downstream-direct-comparisons} compares
\texttt{StaFIR-context} directly with the two fractional benchmarks and a raw
AR(110) reference with matched historical context. For each
series and horizon, the estimand is the paired mean QLIKE difference divided
by the mean QLIKE of that comparator. The four series-level percentages are
then averaged with equal weight. Each contrast is computed from the paired
losses using its comparator's denominator. Against \texttt{FD-context}, only the
$h=5$ interval excludes zero. Against \texttt{FD(0.5)}, no interval excludes
zero. The 95\% coverage is marginal for each horizon--comparator contrast.

\input{generated/tables/downstream_direct_comparisons}

\paragraph{Matched historical context.}
For $K=100$, ten consecutive FIR outputs depend on 110 raw observations.
The \texttt{raw AR(110)} control uses these same 110 observations as separate
ridge predictors, with the stored origin standardizer, training, validation,
and refit target dates, ridge grid, and one-step recursive forecast strategy
held fixed. Its identity kernel requires no reintegration. Thus the comparison
matches available history while allowing different feature dimensions:
110 raw predictors versus ten filtered predictors. Across 234 rolling origins,
both methods score all 80{,}846 forecast--horizon pairs without failures.

Relative to this control, \texttt{StaFIR-context} reduces panel QLIKE by
3.92\% at $h=1$ and 8.24\% at $h=5$, with marginal 95\% intervals
$[-5.35,-2.47]$ and $[-12.03,-4.05]$ percentage points, respectively.
The point estimates also favor StaFIR at $h=20,40,80$, but these intervals
include zero. These interval conclusions hold at expected bootstrap block
lengths 10, 20, 40, and 80 dates. The effect varies across series: CDS favors
StaFIR at every horizon, whereas S\&P~500 point estimates favor AR(110) at
$h=20,40,80$, with intervals including zero. AR(110) itself has higher panel
loss than the ten-lag \texttt{raw} reference by 3.24\% and 6.75\% at $h=1,5$,
with intervals excluding zero. The comparison therefore supports a benefit
from the constrained representation beyond access to additional history;
it does not isolate the contribution of the stationarity penalties or of
mixing multiple decay scales.

\paragraph{Sensitivity to a direct forecasting head.}
\label{ap:direct-head-sensitivity}
We replay the fitted filters on the same four log-realized-variance series
using a separate direct ridge model for each horizon
$h\in\{1,5,20,40,80\}$. Each model predicts the future raw log realized
variance $\ell_{t+h}$ directly from an input consisting of ten transformed
lags. The calibration windows,
training--validation split, ridge grid, and continuous scoring convention
are those of Appendix~\ref{ap:downstream-task}; training targets are kept
within the training segment and all refit targets within the calibration
window.

\input{generated/tables/downstream_direct_head_sensitivity}

Table~\ref{tab:downstream-direct-head-sensitivity} reports paired QLIKE
contrasts for the contextual StaFIR selection. It has lower loss than
\texttt{FD(0.5)} at $h=1,5,20$, but higher loss than \texttt{raw} at
$h=1,5,20,40$. These interval conclusions hold at expected bootstrap block
lengths 10, 20, 40, and 80 dates. At $h=80$, neither contrast excludes zero.
Comparisons with \texttt{FD-context} are less stable: the $h=1$ advantage
includes zero at block length 80, and the small $h=80$ disadvantage includes
zero at block lengths 40 and 80. We therefore interpret this experiment as
sensitivity to the downstream forecasting model.

Reconstruction from the filtered output requires the fitted kernel and the
observed history specified in Appendix~\ref{ap:program_properties}. The direct
model's ten-lag input omits this reconstruction information. The two
experiments therefore assess different uses of the filtered representation:
recursive prediction followed by reintegration, and direct prediction from
transformed lags. Isolating the contribution of reintegration would require
holding the forecasting strategy fixed.

\FloatBarrier
\subsection{Constraint and representation ablations}
\label{ap:constraint-ablations}

\paragraph{Decay scales and second-order penalties.}
Two ablations isolate the contributions of the number of decay scales and of
the local second-order penalties
(Tables~\ref{tab:synthetic-ablation-results}
and~\ref{tab:real-ablation-results}). Restricting \texttt{StaFIR-context} to
seven decay scales changes the synthetic and log-price scores only slightly.
Removing the second-order penalties,
\texttt{StaFIR-context-no-2nd-order}, instead raises
similarity but lowers empirical stationarity and the ADF rejection rate. On the
log-price series it also increases mean effect turnover from 0.725 to 0.984.
The results therefore attribute the main effect to the penalty structure
rather than to the width of the decay-scale basis.

\input{generated/tables/synthetic_ablation_results}

\input{generated/tables/real_ablation_results}

\paragraph{Exponential parameterization against free lag coefficients.}
The two tables also ask whether the exponential parameterization itself
sacrifices empirical quality. The ablation,
\texttt{StaFIR-context-free-lags}, replaces the nine mixture weights with 100
independent nonnegative lag-weight magnitudes while retaining the selected
$(\lambda_s,\lambda)$ of \texttt{StaFIR-context} at every scenario or rolling
origin. The comparison measures the effect of the representation at the
penalties selected for the nine-scale model, rather than penalties tuned for
the larger parameterization. Across the 3{,}000 synthetic scenarios and the
rolling log-price series, its scores, test pass rates, turnover, and fit
success are close to those of the exponential representation.

The larger model is substantially more expensive to estimate. In a matched
re-fit benchmark on the first calibration window of each log-price series,
five timed fits per representation followed one untimed warm-up. The median
fit time with free lag coefficients was 30.4 times that of the exponential
parameterization when pooled over the four series. This timing excludes
candidate selection and artifact I/O.
Once fitted, both representations deploy the same 101-coefficient causal FIR
filter and therefore have the same inference arithmetic and state. The
exponential parameterization is thus a low-dimensional inductive bias. It
preserves empirical quality while reducing the fitted lag parameters from 100
to 9 and the pooled re-fit time by a factor of 30.4. The block compression of
Appendix~\ref{ap:implementation_details} benefits the smaller representation
more strongly.

\FloatBarrier
\subsection{Sensitivity of the selection rule}
\label{ap:selection-sensitivity}

\paragraph{Target fraction and shortfall curvature.}
The main experiments fix $\zeta=0.9$ and $\gamma_{\mathrm{sel}}=0.4$ in the
rule of Appendix~\ref{ap:empirical-rule}. A focused sensitivity experiment
reruns the \texttt{StaFIR-context} selection on the S\&P~500 and Bitcoin
log-price series over
\[
  \zeta\in\{0,0.9,1.0\},
  \qquad
  \gamma_{\mathrm{sel}}\in\{0.2,0.4,0.8\},
\]
holding $c_{\mathrm{sw}}=10$, the candidate grids, the rolling split, and all
other protocol settings fixed. Table~\ref{tab:selection-sensitivity} reports
the nine cells for each series.

\input{generated/tables/selection_sensitivity_results}

\paragraph{Switching cost.}
On the four log-price series, raising $c_{\mathrm{sw}}$ from 0 to 40 lowers
the \texttt{StaFIR-context} selection-change rate from 91.3\% to 12.6\% and
mean effect turnover $T_y$ from 0.824 to 0.620. Mean similarity falls from
84.3 to 73.0 while stationarity rises from 65.6 to 68.3;
$c_{\mathrm{sw}}=10$ provides an intermediate trade-off. The effect of the
selection period $M$ on turnover and cost is reported in
Table~\ref{tab:real-refit-period}.

\FloatBarrier
\subsection{Adaptation cost and selection period}
\label{ap:adaptation-cost}

\paragraph{Measurement boundary.}
\label{ap:cost-measurement}
Table~\ref{tab:cost} separates hyperparameter selection from a carried re-fit.
Each end-to-end adaptation timer starts with a raw calibration window, fits the
training-prefix mean and population standard deviation, applies the resulting
standardization, and then runs the relevant stage. Selection includes every candidate fit and score---all
98 solves for \texttt{StaFIR-context}---and a full-calibration re-fit of the
selected transformation; a carried re-fit re-estimates the weights at fixed
hyperparameters or rebuilds the fractional kernel.
The batch-one inference timer includes fixed standardization, insertion into a
float64 ring buffer, and the 101-coefficient causal FIR dot product for one raw
observation. It excludes construction, adaptation, input preparation, and
artifact I/O. Persistent numeric state comprises 101 coefficients, a
101-sample buffer, and the fitted mean and scale: 204 float64 values, or
1.59\,KiB.

The adaptation resource benchmark uses the same raw-window boundary in a fresh
subprocess. For each stage, resident process memory is reported as the absolute
peak followed by its increase above process start. Wall time measures stage latency;
CPU time measures computational work. For selection period $M$, the reported
cycle cost is
$C_{\mathrm{cycle}}(M)=C_{\mathrm{select}}+(M-1)C_{\mathrm{carried}}$, and the
amortized cost per origin is $C_{\mathrm{cycle}}(M)/M$. Timings are measurements
on the development host under the worker and numerical-library thread settings
recorded with its hardware, operating system, packages, and raw repetitions in
the benchmark artifact.

\input{generated/tables/cost_results}

The remainder of this section amortizes the two adaptation stages over the
selection period.

\paragraph{Selection period on the log-price series.}
Table~\ref{tab:real-refit-period} pairs the held-out rolling log-price
outcomes at each selection period $M\in\{1,2,4,8,\infty\}$ with the amortized
adaptation cost per origin, retaining $c_{\mathrm{sw}}=10$. The fractional
rules are evaluated under the same accounting. Selecting less often lowers
turnover and cost for both contextual methods while leaving similarity almost
unchanged, but freezing the first selection ($M=\infty$) lowers similarity
markedly.

\input{generated/tables/refit_period_results}

\paragraph{Selection period in the forecasting task.}
Table~\ref{tab:downstream-refit-period} reports direct paired stationary-bootstrap QLIKE contrasts between 
selection at every origin and each less frequent schedule. Each schedule varies
the selection frequency of $(\lambda_s,\lambda)$ while re-fitting the
standardization, filter weights, and forecasting model at every origin.
The intervals favor $M=1$ over $M=2$ at
horizons 5 and 20, over $M=4$ at horizons 1--40, over $M=8$ at horizons 20 and 40, 
and over the frozen schedule at horizons 1, 20, and 40. The frozen-schedule interval 
at horizon 5 narrowly includes zero, and all cadence-contrast intervals include
zero at horizon 80. The gains from selection at every origin depend on both
the comparison schedule and the forecast horizon.

\input{generated/tables/downstream_refit_period_results}

%% file: tables/experiments.tex
\begin{table}[t]
\centering
\small
\caption{Implementation details for tests and solver.}
\label{tab:appendix_implementation_details}
\begin{tabularx}{\textwidth}{@{}p{2.2cm}p{3.3cm}p{3.2cm}X@{}}
\toprule
\textbf{Component} &
\textbf{Package} &
\textbf{Routine} &
\textbf{Specification} \\
\midrule

ADF test &
\texttt{statsmodels} v\texttt{0.14.6} &
\texttt{statsmodels.tsa.}\allowbreak\texttt{stattools.adfuller} &
Unit-root null. Constant only, \texttt{regression = "c"}, with one fixed lag, \texttt{maxlag = 1} and \texttt{autolag = None}. Parameters from baseline configuration in \citet{DePrado2019}. We report the statistic, \(p\)-value, and \(5\%\) rejection indicator. \\

\addlinespace

PP test &
\texttt{arch} v\texttt{8.0.0} &
\texttt{arch.unitroot.}\allowbreak\texttt{PhillipsPerron} &
Unit-root null. Constant only, \texttt{trend = "c"}. Bandwidth / lag length is chosen automatically with \texttt{lags = None}. We report the statistic, \(p\)-value, selected bandwidth / lag, and \(5\%\) rejection indicator. \\

\addlinespace

KPSS test &
\texttt{statsmodels} v\texttt{0.14.6} &
\texttt{statsmodels.tsa.}\allowbreak\texttt{stattools.kpss} &
Level-stationarity null. Constant only, \texttt{regression = "c"}. Lag length is selected using \texttt{nlags = "auto"}. We report the statistic, interpolated \(p\)-value, selected lag length, and \(5\%\) rejection indicator. \\

\addlinespace

SOCP solver &
\texttt{cvxpy} v\texttt{1.9.2} with \texttt{clarabel} v\texttt{0.11.1} &
\texttt{CLARABEL} through the \texttt{cvxpy} modeling layer &
Solves the filter SOCP in Section~\ref{sec:methods}. We record solver status, objective value, iterations, feasibility, and runtime. Settings are \texttt{max\_iter = 50} and \texttt{tol\_gap\_abs = tol\_gap\_rel = tol\_feas = }\(10^{-6}\). \\

\bottomrule
\end{tabularx}
\end{table}

%% file: generated/tables/synthetic_results.tex
\begin{table*}[t]
  \caption{Synthetic-data evaluation. The score columns are mean $\pm$ sample SD across independent root seeds after averaging the scenario grid within each seed, and are conditional on successful fits. The three test columns report decision rates over the scenarios at which the test returned a decision, pooled over the whole grid: rejection of the ADF and PP unit-root nulls and non-rejection of the KPSS stationarity null. Fail reports failed fits over attempted fits, and a failed fit is never counted as a failed test. Figure~\ref{fig:synthetic-adaptivity} summarizes the same comparison on the score columns. }
  \label{tab:synthetic-results}
  \centering
  \scriptsize
  \setlength{\tabcolsep}{4pt}
  \begin{tabular}{lrr!{\color{black!60}\vrule width 0.4pt}rrr!{\color{black!60}\vrule width 0.4pt}r}
    \toprule
    & & & \multicolumn{3}{c}{Test outcomes (\%)} & \\
    Method & Stationarity & Similarity & \shortstack{ADF\\rejection} & \shortstack{PP\\rejection} & \shortstack{KPSS\\non-rejection} & Fail \\
    \midrule
    \texttt{raw} & 65.3 $\pm$ 0.6 & -- & 61 & 63 & 19 & 0/3000 \\
    \texttt{1-diff} & 80.4 $\pm$ 0.6 & 50.3 $\pm$ 0.3 & 100 & 100 & 99 & 0/3000 \\
    \texttt{EWMA} & 70.1 $\pm$ 0.5 & 82.5 $\pm$ 0.8 & 85 & 88 & 62 & 0/3000 \\
    \texttt{rolling-z} & 70.0 $\pm$ 0.5 & 83.1 $\pm$ 0.8 & 83 & 86 & 54 & 0/3000 \\
    \texttt{FD(0.5)} & 78.2 $\pm$ 0.5 & 74.4 $\pm$ 0.4 & 100 & 100 & 71 & 0/3000 \\
    \texttt{FD-adf} & 66.2 $\pm$ 0.6 & 99.3 $\pm$ 0.1 & 71 & 74 & 21 & 0/3000 \\
    \texttt{FD-context} & 74.0 $\pm$ 0.8 & 86.0 $\pm$ 1.2 & 95 & 96 & 48 & 0/3000 \\
    \texttt{StaFIR-context} & 73.2 $\pm$ 0.9 & 88.2 $\pm$ 1.1 & 95 & 97 & 31 & 0/3000 \\
    \bottomrule
  \end{tabular}
\end{table*}

%% file: generated/tables/regime_results.tex
\begin{table*}[t]
  \caption{Response to a within-series regime shift. Origins are grouped by whether the rolling
  calibration window is before, across, or after the scheduled break. Entries
  are means over independent seeds with standard errors in parentheses.
  Similarity cost is shown for the transition phase, on which the comparison
  of Appendix~\ref{ap:regime-shift-results} rests. Fixed transformations are
  omitted because they do not recalibrate.
  }
  \label{tab:synthetic-regime-transition}
  \centering
  \scriptsize
  \setlength{\tabcolsep}{3pt}
  \begin{tabular}{lrrr!{\color{black!60}\vrule width 0.4pt}r!{\color{black!60}\vrule width 0.4pt}rrr}
    \toprule
    & \multicolumn{3}{c}{Stationarity gain} & Similarity cost & \multicolumn{3}{c}{Effect turnover $T_y$} \\
    Method & Pre & Trans. & Post & Trans. & Pre & Trans. & Post \\
    \midrule
    \multicolumn{8}{@{}l}{\textbf{Memory increases (\texttt{d\_up})}} \\
    \texttt{FD-adf} & 0.00 (0.00) & 2.91 (0.26) & 3.27 (0.35) & 2.03 (0.18) & 0.000 (0.000) & 0.171 (0.015) & 0.206 (0.017) \\
    \texttt{FD-context} & 4.73 (0.60) & 11.21 (0.43) & 14.62 (0.92) & 19.62 (0.74) & 0.093 (0.023) & 1.159 (0.261) & 1.669 (0.406) \\
    \texttt{StaFIR-context} & 4.60 (0.50) & 9.47 (0.42) & 13.46 (0.97) & 16.08 (0.69) & 0.174 (0.021) & 1.189 (0.192) & 1.325 (0.234) \\
    \midrule
    \multicolumn{8}{@{}l}{\textbf{Memory decreases (\texttt{d\_down})}} \\
    \texttt{FD-adf} & 2.69 (0.37) & 2.97 (0.29) & 0.00 (0.00) & 4.55 (0.59) & 0.164 (0.016) & 0.215 (0.016) & 0.000 (0.000) \\
    \texttt{FD-context} & 13.27 (1.03) & 6.21 (0.47) & 6.38 (0.45) & 10.95 (0.50) & 1.356 (0.296) & 1.278 (0.177) & 0.094 (0.036) \\
    \texttt{StaFIR-context} & 11.72 (0.98) & 5.95 (0.39) & 4.21 (0.44) & 11.81 (0.57) & 1.141 (0.192) & 1.308 (0.241) & 0.260 (0.025) \\
    \midrule
    \multicolumn{8}{@{}l}{\textbf{Volatility persistence increases (\texttt{garch\_up})}} \\
    \texttt{FD-adf} & 0.16 (0.07) & 0.09 (0.04) & 0.06 (0.06) & 0.03 (0.01) & 0.023 (0.008) & 0.014 (0.005) & 0.005 (0.004) \\
    \texttt{FD-context} & 12.07 (0.66) & 7.85 (0.41) & 6.95 (0.46) & 13.82 (0.65) & 0.384 (0.062) & 0.611 (0.068) & 0.604 (0.107) \\
    \texttt{StaFIR-context} & 9.92 (0.52) & 7.12 (0.37) & 6.94 (0.66) & 12.25 (0.60) & 0.494 (0.045) & 0.661 (0.049) & 0.543 (0.071) \\
    \midrule
    \multicolumn{8}{@{}l}{\textbf{Volatility persistence decreases (\texttt{garch\_down})}} \\
    \texttt{FD-adf} & 0.13 (0.07) & 0.06 (0.02) & 0.01 (0.01) & 0.01 (0.01) & 0.026 (0.009) & 0.011 (0.004) & 0.003 (0.002) \\
    \texttt{FD-context} & 6.44 (0.54) & 11.15 (0.34) & 12.02 (0.65) & 16.30 (0.68) & 0.507 (0.079) & 0.516 (0.063) & 0.466 (0.090) \\
    \texttt{StaFIR-context} & 6.16 (0.59) & 9.43 (0.32) & 11.12 (0.70) & 13.60 (0.51) & 0.621 (0.072) & 0.579 (0.056) & 0.565 (0.084) \\
    \bottomrule
  \end{tabular}
\end{table*}

%% file: generated/tables/real_log_price_diagnostics.tex
\begin{table*}[t]
  \caption{Rolling log-price diagnostics pooled over 235 origins. Test
  columns report ADF and PP rejection rates and KPSS non-rejection rates (\%);
  $M,V,G$ are the stationarity components and
  $\rho_{\mathrm P},\rho_{\mathrm P,\mathcal T},\mathrm{dCor}$ the
  similarity components. Aggregate scores are on a 0--100 scale. $\overline
  T_y$ is mean effect turnover over 231 predecessor pairs; ``--'' denotes
  fixed transforms. Pooling is descriptive because origins overlap.
  }
  \label{tab:real-log-price-diagnostics}
  \centering
  \scriptsize
  \setlength{\tabcolsep}{2.6pt}
  \resizebox{\textwidth}{!}{%
  \begin{tabular}{@{}lccc!{\color{black!60}\vrule width 0.4pt}ccc!{\color{black!60}\vrule width 0.4pt}ccc!{\color{black!60}\vrule width 0.4pt}rrr@{}}
    \toprule
    & \multicolumn{3}{c}{Test outcomes (\%)}
    & \multicolumn{3}{c}{Stationarity components}
    & \multicolumn{3}{c}{Similarity components}
    & \multicolumn{2}{c}{Aggregate score}
    & \\
    Method & \shortstack{ADF\\rejection} & \shortstack{PP\\rejection} & \shortstack{KPSS\\non-rejection} & $M$ & $V$ & $G$
      & $\rho_{\mathrm P}$ & $\rho_{\mathrm P,\mathcal T}$ & dCor
      & $S_{\mathrm{stat}}$ & $S_{\mathrm{sim}}$ & $\overline T_y$ \\
    \midrule
    \texttt{raw} & 7 & 6 & 5 & 63.5 & 61.3 & 65.7 & 1.00 & 1.00 & 1.00 & 57.6 & 100.0 & -- \\
    \texttt{1-diff} & 100 & 100 & 97 & 94.3 & 74.9 & 72.3 & 0.07 & 0.19 & 0.16 & 71.9 & 43.0 & -- \\
    \texttt{EWMA} & 42 & 44 & 34 & 70.6 & 64.0 & 65.6 & 0.48 & 0.78 & 0.52 & 61.3 & 71.7 & -- \\
    \texttt{rolling-z} & 30 & 31 & 34 & 70.1 & 65.7 & 66.6 & 0.48 & 0.80 & 0.53 & 62.2 & 72.4 & -- \\
    \texttt{FD(0.5)} & 100 & 100 & 31 & 77.1 & 72.2 & 70.3 & 0.56 & 0.80 & 0.54 & 67.7 & 73.9 & -- \\
    \addlinespace[1pt]
    \texttt{FD-adf} & 38 & 43 & 11 & 66.1 & 64.1 & 66.8 & 0.95 & 0.98 & 0.94 & 59.9 & 96.8 & 0.099 \\
    \texttt{FD-context} & 71 & 72 & 26 & 73.2 & 69.6 & 69.5 & 0.70 & 0.85 & 0.69 & 65.0 & 82.3 & 1.206 \\
    \texttt{StaFIR-context} & 77 & 80 & 10 & 72.5 & 71.5 & 71.0 & 0.73 & 0.84 & 0.72 & 66.0 & 83.5 & 0.718 \\
    \bottomrule
  \end{tabular}}
\end{table*}

%% file: generated/tables/real_paired_comparisons.tex
\begin{table*}[t]
  \caption{Paired rolling log-price comparisons. Entries are mean within-origin
  differences with percentile 95\% stationary-bootstrap intervals
  (10,000 replications; expected block length
  7 origins). Positive stationarity differences favor
  \texttt{StaFIR-context}; positive similarity-cost differences indicate additional
  preservation cost. The final panel gives equal weight to the four series. $\Delta|\widehat d\,|$ is the paired difference in absolute local Whittle residual memory (Table~\ref{tab:real-tests-main}); negative values favor \texttt{StaFIR-context}.
  }
  \label{tab:real-paired-comparisons}
  \centering
  \scriptsize
  \setlength{\tabcolsep}{6pt}
  \begin{tabular}{lrrrr}
    \toprule
    Baseline & $\Delta S_{\mathrm{stat}}$ & $\Delta C_{\mathrm{sim}}$ & $\Delta|\widehat d\,|$ & Paired origins \\
    \midrule
    \multicolumn{5}{@{}l}{\textbf{XO 5Y}} \\
    \texttt{FD-context} & 0.9 [-0.9, 2.5] & -2.2 [-7.0, 2.4] & -0.08 [-0.15, -0.02] & 53 \\
    \texttt{FD(0.5)} & -2.2 [-4.4, -0.3] & -9.3 [-13.3, -5.1] & 0.03 [-0.04, 0.11] & 53 \\
    \texttt{1-diff} & -6.3 [-8.4, -4.3] & -40.5 [-44.2, -36.9] & 0.38 [0.32, 0.46] & 53 \\
    \texttt{FD-adf} & 8.5 [6.3, 10.6] & 15.9 [12.1, 19.7] & -0.41 [-0.51, -0.32] & 53 \\
    \texttt{raw} & 9.8 [7.9, 11.6] & 16.8 [13.1, 20.4] & -0.48 [-0.55, -0.40] & 53 \\
    \midrule
    \multicolumn{5}{@{}l}{\textbf{EUR/USD}} \\
    \texttt{FD-context} & 1.4 [0.4, 2.5] & 0.6 [-1.2, 2.8] & -0.11 [-0.17, -0.06] & 76 \\
    \texttt{FD(0.5)} & -1.1 [-2.7, 0.5] & -5.6 [-10.0, -1.4] & -0.01 [-0.08, 0.07] & 76 \\
    \texttt{1-diff} & -6.0 [-8.5, -4.0] & -36.9 [-41.3, -32.7] & 0.41 [0.33, 0.49] & 76 \\
    \texttt{FD-adf} & 8.1 [6.5, 9.6] & 17.4 [13.3, 21.3] & -0.38 [-0.44, -0.31] & 76 \\
    \texttt{raw} & 10.4 [8.5, 12.1] & 20.1 [15.3, 24.5] & -0.50 [-0.58, -0.42] & 76 \\
    \midrule
    \multicolumn{5}{@{}l}{\textbf{S\&P 500}} \\
    \texttt{FD-context} & 0.2 [-0.3, 0.8] & -3.5 [-6.6, -0.9] & -0.00 [-0.04, 0.05] & 73 \\
    \texttt{FD(0.5)} & -2.6 [-4.5, -1.1] & -17.2 [-20.9, -13.6] & 0.24 [0.16, 0.32] & 73 \\
    \texttt{1-diff} & -5.4 [-8.5, -3.0] & -46.3 [-50.0, -42.3] & 0.57 [0.47, 0.66] & 73 \\
    \texttt{FD-adf} & 2.0 [0.9, 3.4] & 5.7 [1.8, 9.6] & -0.09 [-0.18, -0.01] & 73 \\
    \texttt{raw} & 3.1 [1.8, 4.8] & 9.9 [6.3, 13.7] & -0.28 [-0.36, -0.20] & 73 \\
    \midrule
    \multicolumn{5}{@{}l}{\textbf{Bitcoin (XBX)}} \\
    \texttt{FD-context} & 1.7 [0.2, 3.2] & 1.1 [-1.9, 4.6] & -0.12 [-0.17, -0.07] & 33 \\
    \texttt{FD(0.5)} & -0.1 [-2.9, 2.5] & -2.2 [-8.5, 4.1] & -0.05 [-0.15, 0.06] & 33 \\
    \texttt{1-diff} & -6.4 [-9.1, -3.9] & -35.9 [-42.6, -29.8] & 0.43 [0.32, 0.56] & 33 \\
    \texttt{FD-adf} & 6.7 [2.8, 10.2] & 16.3 [9.1, 23.0] & -0.31 [-0.42, -0.19] & 33 \\
    \texttt{raw} & 12.9 [7.5, 17.7] & 22.5 [15.3, 29.0] & -0.51 [-0.62, -0.37] & 33 \\
    \midrule
    \multicolumn{5}{@{}l}{\textbf{Equal-weight four-series panel}} \\
    \texttt{FD-context} & 1.1 [0.4, 1.7] & -1.0 [-2.7, 0.7] & -0.08 [-0.11, -0.05] & 235 \\
    \texttt{FD(0.5)} & -1.5 [-2.5, -0.5] & -8.6 [-10.9, -6.2] & 0.05 [0.01, 0.09] & 235 \\
    \texttt{1-diff} & -6.0 [-7.3, -4.8] & -39.9 [-42.3, -37.6] & 0.45 [0.40, 0.50] & 235 \\
    \texttt{FD-adf} & 6.3 [5.1, 7.5] & 13.8 [11.4, 16.2] & -0.30 [-0.34, -0.25] & 235 \\
    \texttt{raw} & 9.0 [7.5, 10.5] & 17.3 [14.9, 19.7] & -0.44 [-0.49, -0.40] & 235 \\
    \bottomrule
  \end{tabular}
\end{table*}

%% file: generated/tables/downstream_direct_comparisons.tex
\begin{table*}[t]
  \caption{Paired downstream QLIKE comparisons with recursive reintegration for \texttt{StaFIR-context}
  (\%, equal-weight mean over four series with paired
  stationary-bootstrap marginal 95\% intervals). Each percentage uses the named
  comparator's QLIKE as its denominator. Negative values favor
  \texttt{StaFIR-context}. Bold intervals exclude zero. The raw AR(110) control matches the $K+10=110$ observations of raw history used by the $K=100$ FIR and its ten-lag head.
  }
  \label{tab:downstream-direct-comparisons}
  \centering
  \scriptsize
  \setlength{\tabcolsep}{4pt}
  \begin{tabular}{@{}rrrr@{}}
    \toprule
    Horizon & vs. \texttt{FD-context} & vs. \texttt{FD(0.5)} & vs. \texttt{raw AR(110)} \\
    \midrule
    $h=1$ & $-0.49\;[-0.98,\phantom{-}0.02]$ & $\phantom{-}0.06\;[-0.91,\phantom{-}1.06]$ & $\mathbf{-3.92\;[-5.35,-2.47]}$ \\
    $h=5$ & $\mathbf{-2.45\;[-4.54,-0.30]}$ & $-2.64\;[-6.15,\phantom{-}1.20]$ & $\mathbf{-8.24\;[-12.03,-4.05]}$ \\
    $h=20$ & $-2.54\;[-5.38,\phantom{-}0.31]$ & $-0.36\;[-5.53,\phantom{-}5.47]$ & $-4.46\;[-11.01,\phantom{-}1.98]$ \\
    $h=40$ & $-3.40\;[-6.32,\phantom{-}0.05]$ & $-2.41\;[-7.68,\phantom{-}5.20]$ & $-3.07\;[-9.15,\phantom{-}3.53]$ \\
    $h=80$ & $-2.72\;[-5.38,\phantom{-}0.46]$ & $-2.35\;[-6.76,\phantom{-}3.61]$ & $-2.43\;[-7.40,\phantom{-}2.40]$ \\
    \bottomrule
  \end{tabular}
\end{table*}

%% file: generated/tables/downstream_direct_head_sensitivity.tex
\begin{table*}[t]
  \caption{Direct-head sensitivity: QLIKE differences for
  \texttt{StaFIR-context} versus comparators (\%, marginal 95\%
  stationary-bootstrap intervals). Four series-level percentages, normalized
  by each comparator's mean QLIKE, are averaged equally.
  Negative values favor \texttt{StaFIR-context}; bold intervals exclude zero.
  Forecasts use ten transformed lags to predict raw log realized variance.}
  \label{tab:downstream-direct-head-sensitivity}
  \centering
  \footnotesize
  \setlength{\tabcolsep}{4pt}
  \renewcommand{\arraystretch}{0.85}
  \begin{tabular}{@{}rrrr@{}}
    \toprule
    Horizon & vs. \texttt{raw} & vs. \texttt{FD(0.5)} & vs. \texttt{FD-context} \\
    \midrule
    $h=1$ & $\boldsymbol{+25.51\;[+17.98,+35.07]}$ & $\boldsymbol{-43.84\;[-49.11,-37.47]}$ & $\boldsymbol{-7.18\;[-12.24,-1.54]}$ \\
    $h=5$ & $\boldsymbol{+8.35\;[+5.03,+12.19]}$ & $\boldsymbol{-15.91\;[-20.99,-10.17]}$ & $+0.62\;[-3.01,+4.50]$ \\
    $h=20$ & $\boldsymbol{+5.93\;[+3.35,+9.31]}$ & $\boldsymbol{-9.42\;[-13.52,-5.45]}$ & $-0.23\;[-2.08,+1.83]$ \\
    $h=40$ & $\boldsymbol{+4.39\;[+1.87,+7.35]}$ & $-3.83\;[-8.30,+0.21]$ & $+1.68\;[-0.14,+3.32]$ \\
    $h=80$ & $+0.57\;[-0.96,+1.86]$ & $-1.31\;[-3.76,+0.92]$ & $\boldsymbol{+0.85\;[+0.003,+1.87]}$ \\
    \bottomrule
  \end{tabular}
  \par\smallskip
  \begin{minipage}{0.95\textwidth}
  \scriptsize
  Intervals use 10{,}000 paired circular stationary-bootstrap resamples with
  expected block length 20 dates. The comparator denominator is recomputed
  in each draw. Intervals have marginal coverage for each horizon--comparator
  contrast, conditional on the four observed series. The positive
  lower bound $+0.003$ at $h=80$ against \texttt{FD-context} is shown to three
  decimals to avoid rounding it to zero.
  \end{minipage}
\end{table*}

%% file: generated/tables/synthetic_ablation_results.tex
\begin{table*}[t]
  \caption{Synthetic constraint-ablation evaluation. Scores are mean $\pm$ sample SD across seeds after averaging the scenario grid; test columns report pooled ADF and PP rejection rates and KPSS non-rejection rates (\%), conditional on successful fits. Fail counts unsuccessful fits over attempts. }
  \label{tab:synthetic-ablation-results}
  \centering
  \scriptsize
  \setlength{\tabcolsep}{4pt}
  \begin{tabular}{lrr!{\color{black!60}\vrule width 0.4pt}rrr!{\color{black!60}\vrule width 0.4pt}r}
    \toprule
    & & & \multicolumn{3}{c}{Test outcomes (\%)} & \\
    Method & Stationarity & Similarity & \shortstack{ADF\\rejection} & \shortstack{PP\\rejection} & \shortstack{KPSS\\non-rejection} & Fail \\
    \midrule
    \texttt{StaFIR-context} & 73.2 $\pm$ 0.9 & 88.2 $\pm$ 1.1 & 95 & 97 & 31 & 0/3000 \\
    \texttt{StaFIR-context-7-scale} & 73.3 $\pm$ 0.9 & 88.1 $\pm$ 1.1 & 96 & 97 & 32 & 0/3000 \\
    \texttt{StaFIR-context-free-lags} & 73.1 $\pm$ 0.9 & 88.1 $\pm$ 1.1 & 95 & 97 & 34 & 0/3000 \\
    \texttt{StaFIR-context-no-2nd-order} & 70.7 $\pm$ 0.8 & 91.6 $\pm$ 1.1 & 87 & 90 & 38 & 0/3000 \\
    \bottomrule
  \end{tabular}
\end{table*}

%% file: generated/tables/real_ablation_results.tex
\begin{table*}[t]
  \caption{Rolling log-price constraint ablations. Score and turnover entries
  are the equal-weight mean $\pm$ sample SD of the four per-series rolling
  means. Test entries are equal-weight means of the four per-series ADF
  unit-root rejection rates and KPSS stationarity non-rejection rates.
  Fail counts failed fits over all 235 attempted origins.
  }
  \label{tab:real-ablation-results}
  \centering
  \scriptsize
  \setlength{\tabcolsep}{4pt}
  \begin{tabular}{lrrr!{\color{black!60}\vrule width 0.4pt}rr!{\color{black!60}\vrule width 0.4pt}r}
    \toprule
    & & & & \multicolumn{2}{c}{Test outcomes (\%)} & \\
    Method & Stationarity & Similarity & Effect turnover $T_y$ & \shortstack{ADF\\rejection} & \shortstack{KPSS\\non-rejection} & Fail \\
    \midrule
    \texttt{StaFIR-context} & 66.3 $\pm$ 3.5 & 82.7 $\pm$ 5.5 & 0.725 $\pm$ 0.158 & 78 & 10 & 0/235 \\
    \texttt{StaFIR-context-no-2nd-order} & 64.4 $\pm$ 3.3 & 85.4 $\pm$ 5.2 & 0.984 $\pm$ 0.245 & 62 & 14 & 0/235 \\
    \texttt{StaFIR-context-7-scale} & 66.4 $\pm$ 3.5 & 82.3 $\pm$ 5.5 & 0.715 $\pm$ 0.171 & 80 & 17 & 0/235 \\
    \texttt{StaFIR-context-free-lags} & 66.1 $\pm$ 3.6 & 82.7 $\pm$ 5.4 & 0.734 $\pm$ 0.145 & 78 & 13 & 0/235 \\
    \bottomrule
  \end{tabular}
\end{table*}

%% file: generated/tables/selection_sensitivity_results.tex
\begin{table*}[t]
  \caption{Sensitivity of the contextual selection rule to the reference fraction
  $\zeta$ and the shortfall curvature $\gamma_{\mathrm{sel}}$. Each row
  averages the successful rolling origins of the indicated series. The asterisk marks the
  configuration used in the main experiments,
  $(\zeta,\gamma_{\mathrm{sel}})=(0.9,0.4)$. Stationarity and similarity
  are 0--100 scores. Joint test outcome is the share of origins with rejection
  of both ADF and PP unit-root nulls and non-rejection of the KPSS null.
  $T_y$ is mean effect turnover. Reference reached is the share of origins
  at which the selected candidate's validation score is at least $S_{\min}$.
  Score shortfall is the mean of
  $[S_{\min}-S_{\mathrm{stat}}(y^{\mathrm{val}})]_+$.
  }
  \label{tab:selection-sensitivity}
  \centering
  \scriptsize
  \setlength{\tabcolsep}{4pt}
  \begin{tabular}{@{}lrrrrrrrr@{}}
    \toprule
    & $\zeta$ & $\gamma_{\mathrm{sel}}$ & $S_{\mathrm{stat}}$ & $S_{\mathrm{sim}}$ & \shortstack{Joint test\\outcome} & $T_y$ & \shortstack{Reference\\reached} & \shortstack{Score\\shortfall} \\
    \midrule
    \multicolumn{9}{@{}l}{\textit{S\&P 500}} \\
     & 0.00 & 0.20 & 58.6 & 100.0 & 0.00 & 0.019 & 0.93 & 0.00 \\
     & 0.00 & 0.40 & 58.6 & 100.0 & 0.00 & 0.019 & 0.93 & 0.00 \\
     & 0.00 & 0.80 & 58.6 & 100.0 & 0.00 & 0.019 & 0.93 & 0.00 \\
     & 0.90 & 0.20 & 61.3 & 91.3 & 0.04 & 0.742 & 0.16 & 3.79 \\
    $\ast$ & 0.90 & 0.40 & 61.7 & 90.1 & 0.04 & 0.900 & 0.18 & 2.40 \\
     & 0.90 & 0.80 & 61.9 & 88.4 & 0.08 & 1.001 & 0.22 & 1.36 \\
     & 1.00 & 0.20 & 61.5 & 91.1 & 0.04 & 0.816 & 0.14 & 4.21 \\
     & 1.00 & 0.40 & 61.9 & 88.5 & 0.10 & 0.942 & 0.10 & 2.43 \\
     & 1.00 & 0.80 & 62.4 & 86.3 & 0.12 & 1.191 & 0.15 & 1.50 \\
    \multicolumn{9}{@{}l}{\textit{Bitcoin (XBX)}} \\
     & 0.00 & 0.20 & 56.2 & 100.0 & 0.00 & 0.016 & 1.00 & 0.00 \\
     & 0.00 & 0.40 & 56.2 & 100.0 & 0.00 & 0.016 & 1.00 & 0.00 \\
     & 0.00 & 0.80 & 56.2 & 100.0 & 0.00 & 0.016 & 1.00 & 0.00 \\
     & 0.90 & 0.20 & 66.2 & 84.7 & 0.06 & 0.768 & 0.06 & 3.99 \\
    $\ast$ & 0.90 & 0.40 & 69.0 & 77.5 & 0.12 & 0.818 & 0.18 & 1.72 \\
     & 0.90 & 0.80 & 68.6 & 77.0 & 0.18 & 1.501 & 0.24 & 1.19 \\
     & 1.00 & 0.20 & 67.6 & 80.0 & 0.12 & 0.806 & 0.03 & 4.21 \\
     & 1.00 & 0.40 & 68.7 & 77.1 & 0.18 & 0.959 & 0.09 & 2.95 \\
     & 1.00 & 0.80 & 69.1 & 73.8 & 0.21 & 1.226 & 0.18 & 1.76 \\
    \bottomrule
  \end{tabular}
\end{table*}

%% file: generated/tables/cost_results.tex
\begin{table*}[!ht]
  \caption{Deployment cost of the adaptive representations. Select and carried are median wall times for
  hyperparameter selection and a fixed-setting re-fit; inference is single-update
  p50/p95 latency. FIR path and numeric state are shared.
   CPU times are in seconds; peak RSS is absolute memory/increase above process start. Failures count unsuccessful selections.
  }
  \label{tab:cost}
  \centering
  \tiny
  \setlength{\tabcolsep}{2.5pt}
  \begin{tabular}{lrr!{\color{black!60}\vrule width 0.4pt}rr!{\color{black!60}\vrule width 0.4pt}rrrr!{\color{black!60}\vrule width 0.4pt}rr!{\color{black!60}\vrule width 0.4pt}rr}
    \toprule
    & \multicolumn{2}{c}{Wall adaptation} & Host inf. & State & \multicolumn{4}{c}{Adaptation CPU (s)} & \multicolumn{2}{c}{Peak RSS abs./+$\Delta$ (MiB)} & \multicolumn{2}{c}{Selection failures} \\
    Method & select (s) & carried (ms) & p50/p95 ($\mu$s) & (KiB) & select & carried & cycle & /origin & select & carried & Synthetic & Real \\
    \midrule
    \texttt{StaFIR-context} & 1.56 & 203 & 4.0/4.1 & 1.59 & 1.562 & 0.219 & 1.562 & 1.562 & 200.3/+15.96 & 196.6/+12.38 & 0/3000 & 0/469 \\
    \texttt{FD-adf} & 0.061 & $<$1 & 4.0/4.1 & 1.59 & 0.062 & 0.000 & 0.062 & 0.062 & 185.9/+1.00 & 184.9/+0.05 & 0/3000 & 0/469 \\
    \texttt{FD-context} & 0.12 & $<$1 & 4.0/4.1 & 1.59 & 0.125 & 0.000 & 0.125 & 0.125 & 188.2/+3.70 & 184.7/+0.04 & 0/3000 & 0/469 \\
    \bottomrule
  \end{tabular}
\end{table*}

%% file: generated/tables/refit_period_results.tex
\begin{table*}[t]
  \caption{Adaptation frequency across selection periods. Hyperparameters are selected every $M$ origins ($M{=}\infty$: first origin only); standardization and filters are re-fitted at every origin. Per-origin refit is amortized time; similarity and $T_y$ are mean $\pm$ sample SD over successful origins. Joint outcome requires ADF/PP rejection and KPSS non-rejection; selection change counts hyperparameter changes between consecutive origins (both percentages). }
  \label{tab:real-refit-period}
  \centering
  \scriptsize
  \setlength{\tabcolsep}{4pt}
  \begin{tabular}{lr!{\color{black!60}\vrule width 0.4pt}rr!{\color{black!60}\vrule width 0.4pt}rr}
    \toprule
    & & & & \multicolumn{2}{c}{Share of origins (\%)} \\
    Selection period & Per-origin refit (s) & Similarity & Effect turnover $T_y$ & \shortstack{Joint test\\outcome} & Selection change \\
    \midrule
    \multicolumn{6}{@{}l}{\textbf{\texttt{FD-adf}}} \\
    $M{=}1$ & 0.06 & 96.8 $\pm$ 3.2 & 0.099 $\pm$ 0.123 & 10 & 72 \\
    $M{=}2$ & 0.03 & 96.7 $\pm$ 3.2 & 0.064 $\pm$ 0.129 & 10 & 37 \\
    $M{=}4$ & 0.02 & 96.7 $\pm$ 3.2 & 0.048 $\pm$ 0.128 & 9 & 22 \\
    $M{=}8$ & 0.01 & 96.8 $\pm$ 3.1 & 0.034 $\pm$ 0.116 & 9 & 12 \\
    $M{=}\infty$ & 0.00 & 94.7 $\pm$ 4.3 & 0.000 $\pm$ 0.000 & 10 & 0 \\
    \midrule
    \multicolumn{6}{@{}l}{\textbf{\texttt{FD-context}}} \\
    $M{=}1$ & 0.12 & 82.3 $\pm$ 15.6 & 1.206 $\pm$ 3.074 & 26 & 55 \\
    $M{=}2$ & 0.06 & 83.4 $\pm$ 14.4 & 0.642 $\pm$ 1.968 & 25 & 30 \\
    $M{=}4$ & 0.03 & 84.8 $\pm$ 13.9 & 0.232 $\pm$ 1.079 & 22 & 12 \\
    $M{=}8$ & 0.02 & 84.3 $\pm$ 16.7 & 0.206 $\pm$ 1.070 & 23 & 8 \\
    $M{=}\infty$ & 0.00 & 69.6 $\pm$ 12.1 & 0.000 $\pm$ 0.000 & 39 & 0 \\
    \midrule
    \multicolumn{6}{@{}l}{\textbf{\texttt{StaFIR-context}}} \\
    $M{=}1$ & 1.56 & 83.5 $\pm$ 15.1 & 0.718 $\pm$ 1.362 & 10 & 47 \\
    $M{=}2$ & 0.88 & 82.4 $\pm$ 15.2 & 0.662 $\pm$ 1.738 & 11 & 29 \\
    $M{=}4$ & 0.54 & 83.4 $\pm$ 14.0 & 0.364 $\pm$ 0.757 & 10 & 13 \\
    $M{=}8$ & 0.37 & 83.5 $\pm$ 15.3 & 0.313 $\pm$ 0.613 & 13 & 10 \\
    $M{=}\infty$ & 0.20 & 69.6 $\pm$ 14.4 & 0.175 $\pm$ 0.177 & 26 & 0 \\
    \bottomrule
  \end{tabular}
\end{table*}

%% file: generated/tables/downstream_refit_period_results.tex
\begin{table*}[!ht]
  \caption{Paired stationary-bootstrap QLIKE contrasts for
  \texttt{StaFIR-context} at $M=1$ versus less frequent selection schedules.
  Entries show the percentage change in mean QLIKE at $M=1$ relative to each
  comparator schedule, averaged equally across series at each forecast horizon $h$.
  Brackets give 95\% intervals;
  negative values favor $M=1$.
  Bold intervals exclude zero.
  }
  \label{tab:downstream-refit-period}
  \centering
  \scriptsize
  \setlength{\tabcolsep}{3pt}
  \resizebox{\textwidth}{!}{%
  \begin{tabular}{lrrrrr}
    \toprule
    Comparator schedule & $h{=}1$ & $h{=}5$ & $h{=}20$ & $h{=}40$ & $h{=}80$ \\
    \midrule
    $M{=}2$ & $-0.20 [-0.47,0.08]$ & $\mathbf{-1.18 [-1.98,-0.37]}$ & $\mathbf{-2.35 [-4.29,-0.73]}$ & $-1.32 [-3.48,0.55]$ & $-0.35 [-2.27,2.04]$ \\
    $M{=}4$ & $\mathbf{-0.48 [-0.83,-0.11]}$ & $\mathbf{-2.17 [-3.18,-1.20]}$ & $\mathbf{-4.35 [-6.90,-2.35]}$ & $\mathbf{-3.25 [-6.34,-0.86]}$ & $-1.95 [-4.53,0.75]$ \\
    $M{=}8$ & $-0.38 [-0.85,0.07]$ & $-0.82 [-2.27,0.73]$ & $\mathbf{-3.26 [-5.64,-0.05]}$ & $\mathbf{-5.51 [-8.44,-0.89]}$ & $-3.38 [-5.88,0.19]$ \\
    $M{=}\infty$ & $\mathbf{-0.93 [-1.34,-0.51]}$ & $-1.95 [-3.78,0.01]$ & $\mathbf{-5.42 [-8.18,-2.78]}$ & $\mathbf{-4.70 [-7.60,-1.98]}$ & $-2.14 [-4.56,0.68]$ \\
    \bottomrule
  \end{tabular}
  }
\end{table*}